\documentclass{article}

\PassOptionsToPackage{numbers, compress}{natbib}

\usepackage[main, final]{neurips_2026}

\usepackage[utf8]{inputenc} 
\usepackage[T1]{fontenc}    
\usepackage{hyperref}       
\usepackage{url}            
\usepackage{booktabs}       
\usepackage{amsfonts}       
\usepackage{nicefrac}       
\usepackage{microtype}      
\usepackage{xcolor}         

\usepackage{graphicx}
\usepackage{subcaption}
\usepackage{caption}
\usepackage{pifont}
\usepackage{wrapfig}
\usepackage{booktabs}
\usepackage{multicol, multirow}
\usepackage{array}
\usepackage{tabularx}
\usepackage{makecell}

\usepackage{algorithm}
\usepackage{algorithmic}

\usepackage{amsmath}
\usepackage{amssymb}
\usepackage{mathtools}
\usepackage{amsthm}

\theoremstyle{plain}
\newtheorem{theorem}{Theorem}[section]

\newtheorem{corollary}[theorem]{Corollary}
\theoremstyle{definition}

\theoremstyle{remark}

\title{PDE-PFN: Prior-Data Fitted Neural PDE Solver}

\newcommand{\affilmark}[1]{\textsuperscript{#1}}

\author{%
\begin{tabular}{c}
\textbf{Jaehyeon Park}\affilmark{1}
\quad
\textbf{Mingu Kang}\affilmark{2}
\quad
\textbf{Dongseok Lee}\affilmark{3}
\quad
\textbf{Woojin Cho}\affilmark{4}
\\
\textbf{Kookjin Lee}\affilmark{5}
\quad
\textbf{Anthony Gruber}\affilmark{6}
\quad
\textbf{Youngjoon Hong}\affilmark{3}\footnotemark[1]
\quad
\textbf{Noseong Park}\affilmark{2}\footnotemark[1]
\\[0.7em]
\affilmark{1}SolverX,
\quad
\affilmark{2}KAIST,
\quad
\affilmark{3}Seoul National University,
\quad
\affilmark{4}TelePIX
\\
\affilmark{5}Arizona State University,
\quad
\affilmark{6}Sandia National Laboratories
\\[0.3em]
\texttt{jaehyeon.park@solverx.ai, sgn08050@kaist.ac.kr, lorafa@snu.ac.kr,}
\\
\texttt{woojin.py@gmail.com, kookjin.lee@asu.edu, adgrube@sandia.gov}
\\
\texttt{hongyj@snu.ac.kr,  noseong@kaist.ac.kr}
\end{tabular}%
}

\begin{document}

\maketitle

\begingroup
\renewcommand{\thefootnote}{\fnsymbol{footnote}}
\footnotetext[1]{Corresponding Authors}
\endgroup

\begin{abstract}
  Motivated by the success of large language models (LLMs) with broad generalizability and robustness to noisy or unreliable pre-training data, we explore in-context prediction of PDE solution fields using approximate pre-training data. In addition, inspired by the posterior predictive mean-based inference mechanism of the in-context learning in prior-data fitted networks (PFNs), we propose PDE-PFN, a prior-data fitted neural solver that approximates the posterior predictive mean of PDE solutions via in-context learning. PDE-PFN builds on a PFN architecture with self- and cross-attention mechanisms of Transformer and is pre-trained on noisy approximate solutions generated by physics-informed neural networks, serving as diverse but not necessarily exact priors. Through experiments on a range of two-dimensional PDEs, we demonstrate that PDE-PFN achieves empirical generalization across heterogeneous equations, robustness under noisy priors, and zero-gradient-update in-context inference capability. PDE-PFN achieves competitive or improved predictive accuracy relative to the evaluated baselines while supporting flexible context and query representations.
\end{abstract}

\begin{figure*}[ht]
  \centering
  \includegraphics[width=0.85\textwidth]{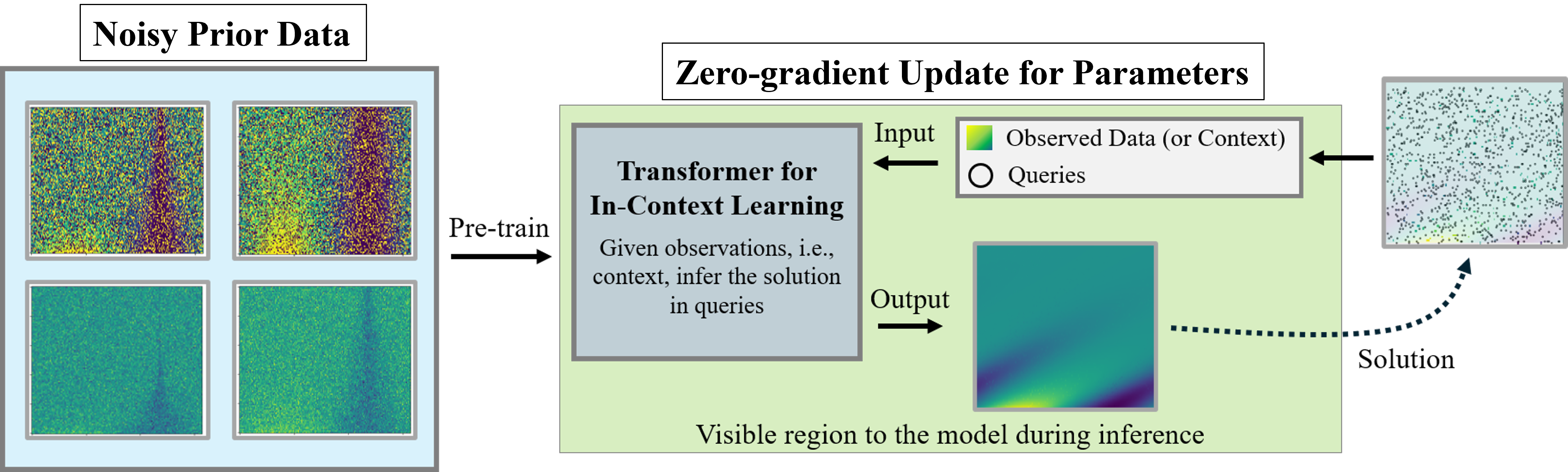}
  \caption{\textbf{The overall workflow of PDE-PFN.} Our model performs in-context learning (ICL) based on given observations (i.e., context) to infer solutions. At inference, PDE-PFN conditions on observations from the target instance without separate operator-level demonstrations or target-instance gradient updates. Even when trained with an approximated PINN prior, our method can improve over noisy approximate pretraining data.}
  \label{fig:Figure1}
  \vskip -0.1 in
\end{figure*}

\section{Introduction}

Partial differential equations (PDEs) have long been at the core of scientific and engineering research, as they provide a mathematical description of the fundamental laws governing diverse natural and engineered systems. With recent advances in machine learning, scientific machine learning (SciML)~\citep{raissi2019physics,willard2022integrating,subramanian2023towards,kim2024simple,choi2024spectral} has emerged as a promising paradigm for solving PDE-related problems.

Within the trend of SciML, there are two conventional approaches: physics-informed neural network (PINN)-based models~\citep{raissi2019physics,shukla2020physics,yang2021b,meng2020ppinn,yuan2022pinn} and operator-based models~\citep{li2021fourier,lu2021learning,tran2023factorized,guibas2021efficient}. PINNs embed a given physical law into a model by minimizing PDE residuals during training, without requiring large labeled datasets. Neural operator-based models learn mappings between infinite-dimensional function spaces given a target operator. While they are mostly general, some operator models require a fixed grid (or a pre-specified coordinate system).

In recent years, large language models (LLMs) have revolutionized natural language processing by introducing highly flexible and scalable architectures~\citep{brown2020language, kaplan2020scaling, touvron2023llama, frieder2023mathematical, chowdhery2023palm}. Motivated by the success of LLMs in enabling broad task generalization, scientific foundation models (SFMs)~\citep{mccabe2023multiple,yang2023context,yang2024pde,hang2025unisolver} have recently been proposed. These models are typically pre-trained on a wide range of PDEs and subsequently fine-tuned for diverse downstream tasks, extending their applicability beyond the narrow scope of conventional neural operators.

Although SFMs represent a meaningful advance by enabling generalization across different PDEs, they continue to inherit other fundamental drawbacks from their backbones. In particular, they lack robustness to noisy input and impose input rigidity, often constraining data to uniform grids or fixed spatial locations. To address these limitations, we draw inspiration from prior-fitted networks (PFNs)~\citep{muller2022transformers}, which implicitly utilizes the posterior predictive distribution (PPD) conditioned on prior information. Building on this idea, we propose \textbf{PDE-PFN}, a new SciML method that adopts a PFN architecture grounded in in-context learning. PDE-PFN uses PFN-style prior fitting to learn context-conditioned predictions, with experiments evaluating within-family task and coefficient generalization, cross-PDE transfer after fine-tuning, and flexible context/query inputs. The PPM perspective characterizes the population target of the MSE objective, while robustness to imperfect priors is evaluated empirically and supported theoretically only under the stated noise assumptions.

\paragraph{Within-family generalization and cross-PDE transfer}
We first evaluated the proposed method on the convection–diffusion–reaction (CDR) equations, a standard benchmark for parameterized PDEs. Our model exhibited strong task generalizability, successfully handling diverse tasks across both parameter and spatiotemporal domains without requiring task-specific fine-tuning. In addition, the CDR-pre-trained model demonstrated cross-PDE transfer after fine-tuning on downstream datasets using a reduced training schedule. Across the evaluated settings, PDE-PFN achieved competitive or improved predictive accuracy relative to task-specific baselines and fine-tuned scientific foundation models.

\paragraph{Flexibility on input shape without prior physical knowledge}
Our method advances SciML by removing two key constraints that limit existing approaches. First, it predicts solutions directly from observed quantities, such as velocity and pressure, without requiring access to symbolic equation information at inference time. This feature is particularly important in real-world applications where the underlying equations are incomplete or entirely unknown~\citep{chien2012manufacturing,rouf2021stock,beck2021perspective,nicolaou2023data,lee2024grid}. Excluding explicit equations from the input is thus an intentional design choice that broadens applicability across diverse domains.
Second, in complex systems such as semiconductor manufacturing, observational data are often collected from irregularly distributed sensors rather than uniform grids~\citep{myers2000improving,quirk2001semiconductor,chien2012manufacturing}. In such settings, the ability to handle inputs of arbitrary shape is essential. Our model accepts variable-sized sets of coordinate–value observations and allows target coordinates to be specified as queries.
Together, these properties support flexible prediction from pointwise observations without requiring a fixed observation grid. A detailed comparison with existing baselines is provided in Table~\ref{table:baseline_comparison} of Appendix~\ref{app:Baselines}. 

\paragraph{Robustness to noisy prior data (or pre-training data)}
For LLMs, one of the most challenging steps is collecting prior data, which typically involves crawling and cleaning sentences from the Internet. However, this process is far from perfect due to two key issues: (i) the Internet, as a data source, is inherently unreliable; and (ii) cleaning such vast amounts of data requires significant manual effort. Consequently, LLMs are often trained on incomplete or imperfect prior data but show surprising outcomes. This realistic yet critical issue has been largely overlooked in the literature on SFMs, despite their similarities to LLMs. In this work, we investigate PFN-style pre-training on approximate PINN solutions and demonstrate strong predictive performance in the evaluated noisy-prior settings.

\paragraph{Zero-gradient-update in-context inference} 
Within the CDR evaluations, a single pre-trained PDE-PFN checkpoint predicts query values from observed context without task-specific gradient updates or fine-tuning. 
The key property is that PDE-PFN adapts to the target PDE instance only through in-context conditioning. 
Compared with in-context operator-learning approaches such as ICON-LM~\citep{yang2025fine}, which require separate task/operator-level ``demos''\footnote{In ICON and ICON-LM, a demo refers to a set of (input, output) pairs of an operator to be inferred.} before prediction, PDE-PFN uses only available observations as the conditioning context and does not require an additional demonstration set or parameter update. 
This makes the inference procedure suitable for settings where irregularly sampled observations are available but retraining or collecting separate demos for each target operator is impractical.


\section{Background}\label{sec:back}
Consider a sequence of pairs $(c_{1},y_{1}),$ $(c_{2},y_{2}),$ $\cdots, $ $(c_{n},y_{n})$, each within the measurable space $(\mathcal{C} \times \mathcal{Y}, \mathcal{A})$, where $c_i$ represents the spatiotemporal coordinate, $y_i$ denotes the corresponding solution, n is data size, and $\mathcal{A}$ denotes the Borel $\sigma$-algebra on the measurable space $\mathcal{C} \times \mathcal{Y}$. These pairs are drawn from a family of probability density distributions $\{p_q : q \in \mathcal{Q}\}$, where $\mathcal{Q}$ represents the \textit{parameter space}, equipped with a $\sigma$-algebra $\mathcal{B}$, ensuring that the mappings $q \mapsto p_q(c, y)$ are measurable. The true parameter $\pi$ is an element of $\mathcal{Q}$, and the pairs $(c_i, y_i)$ are sampled according to $p_\pi$.
Lacking the information on $\pi$, we adopt the following posterior predictive inference framework, which is also used to describe the in-context learning capability of LLMs~\citep{xie2022an}, to establish a prior distribution $\Pi$ which is defined as a probability measure on $(\mathcal{Q}, \mathcal{B})$. Then we have 
\begin{equation}
\Pi(A \mid c,y) = \frac{\int_A p_q(c,y) \mathrm{d}\Pi(q)}{\int_\mathcal{Q} p_q(c,y)\mathrm{d}\Pi(q)},
\end{equation}
for any measurable set $A \in \mathcal{B}$. This prior is updated with the observed data to form the posterior distribution:
\begin{equation}
\Pi(A \mid D_n) = \frac{\int_A L_n(q) \mathrm{d}\Pi(q)}{\int_\mathcal{Q} L_n(q) \mathrm{d}\Pi(q)},
\end{equation}
where $L_n(q) = \prod_{i=1}^n \frac{p_q(c_{i},y_{i})}{p_\pi(c_{i},y_{i})}$ for $A \subset \mathcal{Q}$ and $D_n = \{(c_i, y_i)\}_{i=1}^{n}$. 
The resulting posterior is
\begin{equation}
q_n(c,y \mid D_n) = \int_\mathcal{Q} p_q(c,y) \mathrm{d}\Pi(q \mid D_n),
\end{equation}
and the posterior predictive distribution (PPD) is written as
\begin{equation}
q_n(y \mid c, D_n) = \int_\mathcal{Q} p_q(y \mid c) \, \mathrm{d}\Pi(q \mid D_n).
\end{equation}
As noted by \citep{nagler2023statistical,walker2003sufficient,walker2004modern,walker2004new,de2013bayesian}, for a well-behaved prior, the PPD converges toward true $\pi$ as the data size $n$ increases. This aligns with findings by \citet{de2013bayesian}, demonstrating that in well-specified scenarios, strong consistency is achieved as
\begin{equation}\label{eq:consistency}
\Pi^n \left(\{ q : H(p_\pi, p_q) > \epsilon \}\right) \to 0 \quad \text{almost surely}, 
\end{equation}
for any $\epsilon > 0$,
where $\Pi^n(A) = \int_A d\Pi(q \mid D_n)$ is the posterior measure and $H$ is the Hellinger distance: 
\begin{equation}
H(p, q) = \left(\int_{\mathcal{C}\times\mathcal{Y}} (\sqrt{p} - \sqrt{q})^2 \,\right)^{1/2}.
\end{equation}

\begin{theorem}\label{thm:neural_sensitivity}
There exists a neural network $p_{\hat{\theta}}$ that approximates the posterior predictive distribution $p_\pi(\cdot \mid c)$ (See Appendix~\ref{app:proof} for the full Theorem claim and its proof).
\end{theorem}


Theorem~\ref{app:neural_sensitivity} demonstrates the existence of the neural network $p_{\hat{\theta}}$ that approximates the PPD. While this result justifies learning the full PPD, modeling the entire distribution can be computationally intractable in practice. Since our goal is deterministic PDE prediction, we instead focus on learning the posterior predictive mean (PPM) in the next section. 


\begin{figure*}
  \centering
  \includegraphics[width=0.8\textwidth]{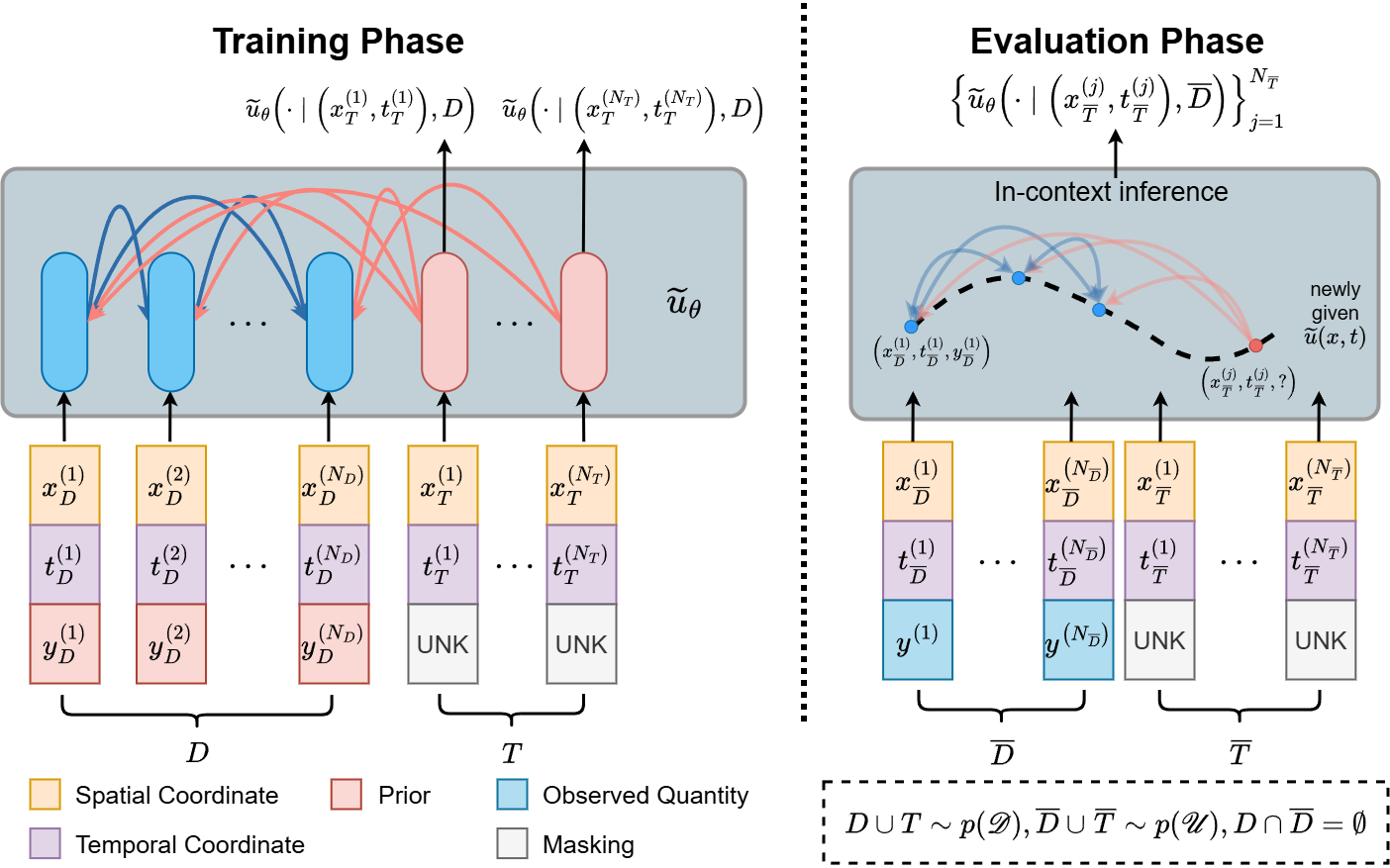}
  \caption{\textbf{A schematic diagram of our model.} \textit{(Left)} Given the contexts $D$ and queries $T$ drawn from the prior distribution $\mathcal{D}$, our model $\Tilde{u}_{\theta}$ is trained to infer the solutions of the queried points $T$ in the training phase. ICL is leveraged with self-attention (blue) within $D$ and cross-attention (red) from $T$ to $D$. \textit{(Right)} In the testing phase, the pre-trained $\Tilde{u}_{\theta}$ takes contexts $\overline{D}$ and queries $\overline{T}$ drawn from the ground truth distribution $\mathcal{U}$, and predicts for the queried points $\overline{T}$.}
  \label{fig:Figure2}
\end{figure*}

\section{Methods} \label{S:Method}
In the Background section, we denoted the spatiotemporal coordinate as $c$ for integration. However, for more clarity, as shown in Figure 2, we separate the spatiotemporal coordinate $c$ into a spatial coordinate $x$ and a temporal coordinate $t$, i.e., $c=(x,t)$. Let $\alpha$ denote a parameter vector representing the coefficients of the governing PDE dynamics. Figure~\ref{fig:Figure2} illustrates the schematic diagram of our model. As defined in Figure~\ref{fig:Figure2}, $D = \{(x_D^{(i)},t_D^{(i)}, y_D^{(i)})\}_{i=1}^{N_{D}}$ and $T = \{(x_T^{(i)},t_T^{(i)}, y_T^{(i)})\}_{i=1}^{N_{T}}$ represent the context set and the query set for ICL, respectively. We assume these sets are independently and identically sampled from a family of probability density distributions $\{p_{\boldsymbol{\alpha}} : \boldsymbol{\alpha} \in \mathcal{Q}\}$. Here, each $y$ corresponds to the quantity at $(x,t)$ for a PDE solution $u(\boldsymbol{\alpha})$ determined by the PDE parameter $\boldsymbol{\alpha}$. Depending on the problem setting, $u$ can be either the ground-truth solution or an approximated solution generated by a PINN, since we consider both exact and approximated prior data. The PPD of the solutions given the context set is written as $\pi(y \mid x,t, D) = \int_{\mathcal{Q}} p_{\boldsymbol{\alpha}} (y \mid x,t) \, \mathrm d\Pi (\boldsymbol{\alpha} \mid D)$ and its posterior predictive mean (PPM), which is our training goal, as
\begin{equation}\label{eqs:ppd} 
\mathbb{E}[Y|x,t,D] = \int_Y y\:\pi(y \mid x,t, D) \, \mathrm dy. 
\end{equation}
For categorical outputs, PFNs can use cross-entropy to approximate posterior predictive class probabilities~\citep{xie2022an,muller2022transformers}. In this work, because PDE-PFN is a deterministic regressor, we minimize the MSE between the predicted solution and the target solution.

\begin{theorem}
Under squared loss, the population-optimal deterministic predictor is the posterior predictive mean conditioned on the available context. Even with the unbiased target noise, the same population minimizer equals the clean-solution PPM conditioned on the noisy context. (See Appendix~\ref{app:proof} for the full Theorem claim and its proof).
\end{theorem}
As shown in Theorem~\ref{thm:mse_ppm}, this objective corresponds to estimating the PPM induced by the prior over PDE solution fields, rather than the full PPD. The theorem also clarifies that, under unbiased target noise, the MSE-optimal deterministic predictor remains the clean-solution conditional mean given the noisy context.


\paragraph{Model architecture} 
Since our model is essentially a prior-data fitted network (PFN), we construct its architecture by modifying certain components of the original PFN. To better handle the more complex PDE data, we first incorporate a Fourier feature embedding into the input data before it is passed to the encoder, extending the original design that was limited to tabular data. We also introduce several modifications to the Transformer architecture. Specifically, to improve the model’s ability to learn from complex data while reducing computational complexity and enhancing representational capacity, we adopt the attentive graph filter (AGF)~\cite{wi2025learning} structure for Transformer layers. A more detailed description of the model architecture can be found in Appendix~\ref{app:model_architecture}.

\paragraph{Benchmark PDEs}
We consider convection-diffusion-reaction (CDR) equations as benchmark PDEs for diverse tasks\footnote{We note that in Eq.~\eqref{eqs:governing_equation} $J$ reaction terms are considered. In previous works~\citep{cho2023hypernetwork, cho2024parameterized, kang2024can}, only one type of reaction term is considered. Eq.~\eqref{eqs:governing_equation} therefore defines a parameterized CDR prior with three selected nonlinear reaction terms for PFN training. Our experiments differ from previous works in this regard.}:
\begin{equation}
\label{eqs:governing_equation}
\begin{aligned}
    u_t = \mathcal{N}(\cdot), \;\;\;  \mathcal{N}(t, x, u, \boldsymbol{\beta}, \boldsymbol{\nu}, \rho_1, \cdots, \rho_J) = -\boldsymbol{\beta} \cdot \nabla u + \nabla \cdot (\boldsymbol{\nu} \nabla u) + \sum_{j=1}^{J} \rho_{j} f_j(u),
\end{aligned}
\end{equation}
where $u$ denotes the state variable; $\boldsymbol{\beta}$ is the convection (advection) velocity; $\boldsymbol{\nu}$ denotes the anisotropic diffusivity coefficient; and $\rho_1, \rho_2, \rho_3$ are scalar coefficients corresponding to the nonlinear reaction terms $f_1, f_2, f_3: \mathbb{R}\rightarrow\mathbb{R}$, respectively. For clarity, we will refer to the two-dimensional CDR equations, including cases where certain coefficients are zero, as the family of CDR equations. This family encompasses not only generic parabolic equations but also extends to hyperbolic equations such as convection equations. Hence, the family of CDR equations contains three key terms—convection, diffusion, and reaction—each with distinct properties, making it an ideal benchmark problem. To the best of our knowledge, this work is the first to address a wide range of reaction terms within the family of CDR equations using a single unified model. We use a set of CDR-related terms and a linear combination of three nonlinear reaction terms to generate prior data. This formulation enables the incorporation of diverse reaction dynamics, making the benchmark problems more comprehensive.

\paragraph{PINN-based approximations of PDE solution space} Although numerical solvers can solve the family of CDR equations, in this work, we use the predictions from PINNs as a representative example of noisy approximated data. This allows us to evaluate the robustness of our model to noise, which is commonly encountered in real-world scenarios. From this point, we refer to these PINN-based approximations as PINN priors. To approximate the solution space for the PDEs, we construct a practical parameter subspace, $\Omega \subseteq \mathcal{Q}$, which is the collection of coefficients in Eq.~\eqref{eqs:governing_equation} and has a dictionary form $\Omega = \{ \boldsymbol{\alpha} := (\boldsymbol{\beta}, \boldsymbol{\nu}, \rho_1, \cdots, \rho_{J}) \}$. Consequently, the target exact prior $\mathcal{U}$ defined below represents the collection of solutions $u(\boldsymbol{\alpha})$ at Eq.~\eqref{eqs:governing_equation} for each parameter $\boldsymbol{\alpha} \in \Omega$, where $\mathcal{X}$ and $\mathcal{T}$ correspond to the spatial and temporal domains of interest, respectively.
\begin{equation}
\begin{aligned}
\mathcal{U} = \bigcup_{\boldsymbol{\alpha} \in \Omega} \{ u(\boldsymbol{\alpha}) \,|\, u_t = \mathcal{N}(t, x, u, \boldsymbol{\alpha}) \}, \;\;\;\; u: \mathcal{X} \times \mathcal{T} \rightarrow \mathbb{R}.
\end{aligned}
\end{equation}
Since the target exact prior data $\mathcal{U}$ is hard to obtain, we instead use a PINN prior $\mathcal{D}$ that closely approximates $\mathcal{U}$ as follows. Suppose $\Tilde{u}(\boldsymbol{\alpha})$ is the prediction by PINN, which is trained to predict the PDE $u_t = \mathcal{N}(\cdot)$ (see Appendix \ref{app:PINN} for details). The PINN prior $\mathcal{D}$ is a collection of approximate solutions $\Tilde{u}(\boldsymbol{\alpha})$ for each $\boldsymbol{\alpha} \in \Omega$,
\begin{equation}
\begin{aligned}
\mathcal{D} = \bigcup_{\boldsymbol{\alpha} \in \Omega} \{ \tilde{u}(\boldsymbol{\alpha}) \}, \;\;\;\; p(\mathcal{D}) \sim p(\mathcal{U}).
\end{aligned}
\end{equation}
Subsequently, the model performs the posterior predictive mean inference conditioned on the generated prior $p(\mathcal{D})$ through ICL.

\paragraph{Objective function} From a given parameter space $\Omega$, the parameter $\boldsymbol{\alpha}$ is randomly drawn i.i.d. from $\Omega$. After this, the corresponding prior $\Tilde{u}(\boldsymbol{\alpha})$ is given as an input to our model $\Tilde{u}_{\theta}$ to minimize the MSE at the queried points, as expressed in Eq.~\eqref{eqs:train_loss1}, which leads to the posterior predictive mean inference~\citep{muller2022transformers} --- in other words, the MSE loss is enough to approximate the mean of Eq.~\eqref{eqs:ppd} without knowing its uncertainty. The specific MSE loss chosen to regress the solution over the spatial and temporal domain for a given $\Tilde{u}(\boldsymbol{\alpha})$ is defined as 
\begin{equation}
\mathcal{L}_{\boldsymbol{\alpha}} = \frac{1}{N_T} \sum_{j=1}^{N_T} \left[ \tilde{u}_{\theta}(x_T^{(j)}, t_T^{(j)} \mid D) - \tilde{u}(\boldsymbol{\alpha})(x_T^{(j)}, t_T^{(j)})\right]^2.
\label{eqs:train_loss1}
\end{equation}
In addition, we introduce a regularization term into the objective function to ensure orthogonality during the SVD process within the AGF layer (see Appendix~\ref{app:model_architecture} for detail). We denote this term as $\mathcal{L}_{AGF}$. To balance the trade-off between the loss and the regularization, we scale the latter by the number of context points $N_D$. Consequently, the final objective function $\mathcal{L}$ is expressed as
\begin{equation}
\mathcal{L} = \mathcal{L}_{\boldsymbol{\alpha}} + \frac{1}{N_D} \mathcal{L}_{AGF}.
\label{eqs:train_loss2}
\end{equation}
The concrete flow of the training and the evaluation process is described in Appendix~\ref{app:Training Algorithm} and~\ref{app:Evaluation}.


\section{Experiments} \label{S:Experiments}

Our experimental section is divided into three phases. In the first phase, we introduce the dataset and establish the baselines for our experiments. In the second phase, we pre-train our model on the family of 2D CDR equations and evaluate it on a variety of tasks without any additional task-specific fine-tuning. This allows us to examine the model’s ability to handle diverse tasks. In the third phase, we fine-tune our model on different PDEs, including the shallow water equations (SWE) and compressible Navier-Stokes equations (CNSE), to assess cross-PDE transfer after fine-tuning. Table~\ref{table:Roadmap} summarizes the experiments in this section.

\begin{table*}[htbp]
\captionsetup{skip=1.5pt}
\caption{Roadmap of the experiment section. `Inter' and `Extra' denote spatiotemporal interpolation and temporal extrapolation tasks, respectively. The `Coeff/dataset' column specifies either the coefficients of the equation or the dataset used in evaluation. \label{table:Roadmap}}
\small
\centering
\resizebox{0.9\textwidth}{!}{%
\begin{tabular}{lllll}
\toprule
\textbf{Section}       & \textbf{Task}    &    \textbf{Coeff/Dataset}    &   \textbf{Equation}  &   \textbf{Objective}  \\ 
\midrule
\ref{SSS:CDR_seen}    & Inter \& Extra    & Seen        &  \multirow{2}{*}{Family of 2D CDR equations} & Within-family task generalization \\
\ref{SSS:CDR_unseen}  & Inter \& Extra    & Unseen      &    &  Within-family coefficient generalization \\
\midrule
\ref{SSS:SWE}         & Inter \& Extra    & Seen \& Unseen  &  SWE   &  Cross-PDE transfer after fine-tuning \\
\ref{SSS:CNSE}        & Operator Learning & Unseen          &  CNSE  &  Cross-PDE transfer after fine-tuning  \\
\bottomrule
\end{tabular}%
}
\end{table*}


\subsection{Experimental Setup}
\label{SS:Experimental_setup}

\paragraph{Dataset}
We consider a family of 2D CDR equations with three reaction terms as the pre-training dataset. This family is constructed from equations with coefficient dictionary ($\beta_x, \beta_y, \nu_x, \nu_y, \rho_1, \rho_2, \rho_3 \in \{0.0, 0.5, 1.0\}$), resulting in a total of 2,187 unique PDEs. They are simulated from the same initial condition, which allows differences in the resulting dynamics to stem solely from the variations in coefficients. For PINN prior, each PINN is trained independently, and its predicted solutions are used as noisy priors. On average, the generated PINN priors exhibit an $L_2$ relative error of 0.02280. For cross-PDE transfer experiments, we employ the SWE and CNSE datasets from PDEBench~\citep{takamoto2022pdebench}. These two equations are chosen because they are commonly used to validate PDE solvers in SciML. Both datasets consist of 1,000 samples generated from distinct initial conditions. The SWE dataset contains a single feature, the water height, whereas the CNSE dataset includes $x$- and $y$-velocities, pressure, and density. Detailed information on the datasets can be found in Appendix~\ref{app:datasets}.

\paragraph{Baseline methods} 
We compare our model with six baseline models: deep operator network (DeepONet)~\citep{lu2021learning}, Fourier neural operator (FNO)~\citep{li2021fourier}, factorized Fourier neural operator (F-FNO)~\citep{tran2023factorized}, adaptive Fourier neural operator (A-FNO)~\citep{guibas2021efficient}, latent neural operator (LNO)~\citep{wang2024latent}, Poseidon~\citep{herde2024poseidon}, and auto-regressive denoising pre-training operator Transformer (DPOT)~\citep{hao2024dpot}. These baselines are chosen for their scalability and ability to effectively handle complex 2D PDEs. For Poseidon and DPOT, scientific foundation models with publicly available pre-trained weights, we use the released weights and perform fine-tuning. The other baselines are trained and evaluated separately on each dataset and task. For the spatiotemporal interpolation task, DeepONet is the only baseline used, excluding those that require grid inputs. The best hyperparameters for all models are selected based on their best validation performance in each task.
Detailed information on the baselines used in each experiment can be found in Appendix~\ref{app:Baselines}.


\subsection{Pre-training on the Family of 2D CDR Equations and Evaluation}
\label{SS:CDR}

The training dataset is constructed from 6 time steps within the range $t \in [0.0, 0.5]$ at intervals of 0.1. Our model takes as an input set of contexts, which include domain coordinates paired with their solution values, along with queries, which consist of domain coordinates where the solution values are masked. We denote the version of our model pre-trained on numeric priors as \textbf{Ours} and the version pre-trained on PINN priors as \textbf{Ours (PINN)}. For the baselines, the training procedure differs by task. In the spatiotemporal interpolation task, since DeepONet requires fixed points as branch input, we predefined the fixed points in training data by sampling and the remaining points are used as queries to be predicted. In the temporal extrapolation task, the baselines are trained to take the grid points at a time $t_n$ as input and predict the grid points at the subsequent time step $t_{n+1}$.
Further information on the training procedure, such as the number of epochs, can be found in Appendix~\ref{app:experimental_details}.


\subsubsection{Task Generalization for Seen coefficients}
\label{SSS:CDR_seen}

We first evaluate the task generalization ability in spatiotemporal interpolation and temporal extrapolation without additional task-specific fine-tuning. Both tasks are conducted within the coefficient dictionary used during pre-training, which we refer to as seen coefficients. For the spatiotemporal interpolation task, the baseline DeepONet requires fixed points as branch input. Therefore, our model is also provided with the same set of context points during evaluation to ensure a fair comparison. Specifically, we use the points at $t = 0.25$ as validation queries and the points at $t = 0.05, 0.15, 0.35,$ and $0.45$ as test queries. For the temporal extrapolation task, the points at $t = 0.5$ are used as contexts, and those at $t = 0.6$ as queries for validation. At test time, the points at $t = 0.6$ serve as contexts, and those at $t = 0.7, 0.8, 0.9,$ and $1.0$ are used as queries. The baselines predict a sequence of query points through rollout, whereas our model can directly predict the query points from their coordinates so that it is evaluated in this manner.
The evaluation results are shown in Table~\ref{table:CDR_seen}.

\begin{table*}[ht]
\captionsetup{skip=1.5pt}
\caption{The evaluation results for the spatiotemporal interpolation and temporal extrapolation tasks applied to the family of 2D CDR equations. They are measured at the seen coefficients ($\beta_x, \beta_y, \nu_x, \nu_y, \rho_1, \rho_2, \rho_3 \in \{0.0, 0.5, 1.0\}$). The best performance is marked in \textbf{bold} and the second-best performance is marked with \underline{underline}.}
\label{table:CDR_seen}
\small
\centering
    \centering
    \captionsetup{labelformat=empty}
    \resizebox{1.0\textwidth}{!}{%
    \begin{tabular}{llccccccccc}
    \toprule
    \textbf{Task}  & \textbf{Metric} & \textbf{Ours} & \textbf{Ours (PINN)} &  \textbf{DeepONet} & \textbf{FNO} & \textbf{F-FNO} & \textbf{A-FNO}  & \textbf{LNO}  & \textbf{Poseidon}   & \textbf{DPOT}  \\ 
    \midrule
    \multirow{3}{*}{\shortstack[l]{Spatiotemporal\\interpolation}} & $L_1$ Abs  & \textbf{0.01218}  & \underline{0.01507}  & 0.01922  & \multirow{3}{*}{-} & \multirow{3}{*}{-} & \multirow{3}{*}{-} & \multirow{3}{*}{-} & \multirow{3}{*}{-}  & \multirow{3}{*}{-} \\
     & $L_2$ Rel & \textbf{0.01643}  & \underline{0.02057}  & 0.02479  & & & & & & \\
     & $L_\infty$ Rel & \textbf{0.04600}  & 0.06584  & \underline{0.05980}  & & & & & &  \\
    \midrule
    \multirow{3}{*}{\shortstack[l]{Temporal\\extrapolation}} & $L_1$ Abs  & \textbf{0.01474}  & \underline{0.01940}  & 0.06529 & 0.03302 & 0.02731 & 0.05800 & 0.02876 & 0.14252 &  0.04924 \\
    & $L_2$ Rel & \textbf{0.02261}  & \underline{0.02705}   & 0.08212  & 0.04058 & 0.03430 & 0.07915 & 0.03665 & 0.15580  & 0.08587  \\
    & $L_\infty$ Rel & \textbf{0.08444}  & 0.08802   & 0.22802  & 0.09553 & 0.09877 & 0.28905 & 0.10948 & 0.28285  & \underline{0.08587}  \\
    \bottomrule
    \end{tabular}%
    }
\end{table*}

The evaluation results demonstrate that both versions of our model outperform the baselines on both tasks. This indicates that once trained on a single dataset, our model can achieve strong performance without requiring task-specific fine-tuning. The PINN-prior results support robustness to the approximation errors present in the evaluated pre-training data. Additionally, we conduct experiments where the PINN prior is provided as a context set during evaluation in Appendix~\ref{app:experi_CDR_given_noisy}.

\subsubsection{Task Generalization for Unseen coefficients}
\label{SSS:CDR_unseen}
We further evaluate the task generalization ability on both tasks in parameter space using the family of 2D CDR equations with unseen coefficients, without any additional fine-tuning. The unseen coefficients are categorized into two cases: coefficient interpolation and coefficient extrapolation. Coefficient interpolation experiments are performed on the equations with intermediate coefficient values not seen during pre-training, ($\beta_x, \beta_y, \nu_x, \nu_y, \rho_1, \rho_2, \rho_3 \in \{0.25, 0.75\}$). Coefficient extrapolation experiments, on the other hand, are conducted on the equations with larger coefficient values than those used in pre-training, ($\beta_x, \beta_y, \nu_x, \nu_y, \rho_1, \rho_2, \rho_3 \in \{1.25, 1.5\}$). For evaluation with unseen coefficients in the temporal extrapolation task, we modify the temporal extrapolation setup to remain within the temporal domain used during training: the points at $t = 0.1$ are provided as contexts, and models are asked to predict the values at $t = 0.2, 0.3, 0.4,$ and $0.5$. The evaluation results are summarized in Table~\ref{table:CDR_unseen}.

\begin{table*}[ht]
\captionsetup{skip=1.5pt}
\caption{The evaluation results for the spatiotemporal interpolation and temporal extrapolation tasks applied to the family of 2D CDR equations. They are measured at the unseen coefficients. The best performance is marked in \textbf{bold} and the second-best performance is marked with \underline{underline}.}
\label{table:CDR_unseen}
\small
\centering
    \centering
    \captionsetup{labelformat=empty}
    \resizebox{1.0\textwidth}{!}{%
    \begin{tabular}{lllccccccccc}
    \toprule
    \textbf{Coefficients} & \textbf{Task}  & \textbf{Metric} & \textbf{Ours} & \textbf{Ours (PINN)}  &  \textbf{DeepONet} & \textbf{FNO} & \textbf{F-FNO} & \textbf{A-FNO} & \textbf{LNO} & \textbf{Poseidon}   & \textbf{DPOT}  \\ 
    \midrule
    \multirow{6}{*}{Interpolation} & \multirow{3}{*}{\shortstack[l]{Spatiotemporal\\interpolation}} & $L_1$ Abs  & \textbf{0.01241}  & \underline{0.01277}  & 0.01347  & \multirow{3}{*}{-} & \multirow{3}{*}{-} & \multirow{3}{*}{-} & \multirow{3}{*}{-}  & \multirow{3}{*}{-} & \multirow{3}{*}{-} \\
    & & $L_2$ Rel & \textbf{0.01658}  & 0.01761  & \underline{0.01731}  &  &  &  &  &  &   \\
    & & $L_\infty$ Rel & \underline{0.04274}  &  0.05503  & \textbf{0.04114}  &  &  &  &  &  &   \\
    \cmidrule{2-12}
    & \multirow{3}{*}{\shortstack[l]{Temporal\\extrapolation}} & $L_1$ Abs  & \underline{0.08957}  & \textbf{0.08709}  & 0.11918 & 0.11625 & 0.11194 & 0.12220 & 0.11085  & 0.17740 &  0.13879 \\
    & & $L_2$ Rel  & \textbf{0.11990}  & \underline{0.12465}  & 0.14903  & 0.14658 & 0.14348 & 0.18191 & 0.14913 & 0.21340  & 0.15085  \\
    & & $L_\infty$ Rel  & \underline{0.29526}  &  0.36472 & 0.35196  & \textbf{0.28689} & 0.30813 & 0.64429 & 0.33280 & 0.39890  & 0.15085  \\ 
    \midrule
    \multirow{6}{*}{Extrapolation} & \multirow{3}{*}{\shortstack[l]{Spatiotemporal\\interpolation}} & $L_1$ Abs  & \underline{0.01661}  & \textbf{0.01582}  & 0.03576  & \multirow{3}{*}{-} & \multirow{3}{*}{-} & \multirow{3}{*}{-} & \multirow{3}{*}{-}  & \multirow{3}{*}{-} & \multirow{3}{*}{-} \\
    & & $L_2$ Rel & \underline{0.02398}  & \textbf{0.02038}  & 0.04669   &  &  &  &  &  &   \\
    & & $L_\infty$ Rel & \underline{0.08324}  & \textbf{0.06890}  & 0.10392   &  &  &  &  &  &  \\
    \cmidrule{2-12}
    & \multirow{3}{*}{\shortstack[l]{Temporal\\extrapolation}} & $L_1$ Abs  & \textbf{0.01027}  & \underline{0.02063}  & 0.05514 & 0.02652 & 0.02854 & 0.05013  & 0.05076 & 0.14661 &  0.05186 \\
    & & $L_2$ Rel  & \textbf{0.01068}  & \underline{0.02169}   & 0.06450  & 0.03043 & 0.03300 & 0.06351 & 0.06347 & 0.15650  & 0.05189  \\
    & & $L_\infty$ Rel  & \textbf{0.01661}  & \underline{0.03474}   & 0.16555  & 0.06124 & 0.08343 & 0.20955 & 0.15091 & 0.27502  & 0.05189  \\
    \bottomrule
    \end{tabular}%
    }
\end{table*}
The evaluation results show that both versions of our model outperform the baselines on both tasks under coefficient interpolation and extrapolation. This shows that our models can make accurate predictions not only on PDEs with coefficients seen during training but also on PDEs with unseen coefficients, without requiring additional fine-tuning. In other words, the models can generalize across the parameter space, achieving both interpolation and extrapolation.


\subsection{Cross-PDE Transfer after Fine-tuning} \label{SS:diverse_PDE}

To evaluate cross-PDE transfer after fine-tuning, we conduct experiments on heterogeneous PDEs beyond the 2D CDR family. On the shallow water equations (SWE) dataset, we assessed whether our model could adapt through fine-tuning to other PDEs on the same tasks as in the family, namely spatiotemporal interpolation and temporal extrapolation. The baselines for each task are trained in the same manner as in the family.
On the CNSE dataset, we further examined whether the model could adapt to trajectory prediction for a different PDE.

\subsubsection{Cross-PDE Transfer to Shallow Water Equations}
\label{SSS:SWE}

To evaluate the PDE generalization ability of our model, we conduct experiments on the SWE, assessing whether the tasks defined in the family of 2D CDR equations can be successfully transferred to this heterogeneous PDE using half as many epochs as the task-specific baselines. Out of the total 1,000 samples, 700 are used for training, and the training, validation, and test datasets are constructed over the time interval in the same manner as for the family of 2D CDR equations. The test dataset generated from the initial condition samples used during training is referred to as the seen dataset, while the test dataset generated from the remaining 300 samples is referred to as the unseen dataset. For the baselines, training for each task is carried out in the same manner as in pre-training.

\begin{table*}[ht]
\captionsetup{skip=1.5pt}
\caption{The evaluation results for the spatiotemporal interpolation and temporal extrapolation tasks applied to SWE. They are measured at the seen/unseen datasets. The best performance is marked in \textbf{bold} and the second-best performance is marked with \underline{underline}.}
\label{table:SWE}
\small
\centering
    \centering
    \captionsetup{labelformat=empty}
    \resizebox{1.0\textwidth}{!}{%
    \begin{tabular}{lllccccccccc}
    \toprule
    \textbf{\shortstack[l]{Evaluation\\dataset}} & \textbf{Task}  & \textbf{Metric} & \textbf{Ours} & \textbf{Ours (PINN)}  &  \textbf{DeepONet} & \textbf{FNO} & \textbf{F-FNO} & \textbf{A-FNO} & \textbf{LNO} & \textbf{Poseidon}   & \textbf{DPOT}  \\ 
    \midrule
    \multirow{6}{*}{Seen} & \multirow{3}{*}{\shortstack[l]{Spatiotemporal\\interpolation}} & $L_1$ Abs  & \textbf{0.00208}  & \underline{0.00269}  & 0.00271  & \multirow{3}{*}{-} & \multirow{3}{*}{-} & \multirow{3}{*}{-} & \multirow{3}{*}{-}  & \multirow{3}{*}{-} & \multirow{3}{*}{-} \\
    & & $L_2$ Rel & \textbf{0.00739}  & \underline{0.00822}  & 0.00997  & & & & & &  \\
    & & $L_\infty$ Rel & \textbf{0.05897}  & \underline{0.06950}  & 0.06951  & & & & & &  \\
    \cmidrule{2-12}
    & \multirow{3}{*}{\shortstack[l]{Temporal\\extrapolation}} & $L_1$ Abs  & \underline{0.01548}  & \textbf{0.01418}  & 0.06994 & 0.03271 & 0.02292 & 0.05036 & 0.02902 & 0.02964 &  0.03327 \\
    & & $L_2$ Rel  & \underline{0.03273}  & \textbf{0.03221}  & 0.08319  & 0.05429 & 0.04110 & 0.07436 & 0.05269 & 0.05208  & 0.05860  \\	
    & & $L_\infty$ Rel  &  \textbf{0.20590} & 0.24596  & 0.28736  & 0.24954 & \underline{0.24478} & 0.38030 & 0.24832 & 0.25668  & 0.29727  \\	
    
    \midrule
    \multirow{6}{*}{Unseen} & \multirow{3}{*}{\shortstack[l]{Spatiotemporal\\interpolation}} & $L_1$ Abs  & \textbf{0.00213}  & \underline{0.00272}  & 0.00278  & \multirow{3}{*}{-} & \multirow{3}{*}{-} & \multirow{3}{*}{-} & \multirow{3}{*}{-} & \multirow{3}{*}{-}  & \multirow{3}{*}{-} \\   
    & & $L_2$ Rel & \textbf{0.00754}  & \underline{0.00828}  & 0.01008  & & & & & &  \\
    & & $L_\infty$ Rel & \textbf{0.05986}  &  \underline{0.06981}   & 0.07047  & & & & & &  \\
    \cmidrule{2-12}
    & \multirow{3}{*}{\shortstack[l]{Temporal\\extrapolation}} & $L_1$ Abs  & \underline{0.01571}  & \textbf{0.01448}  & 0.07047 & 0.03342 & 0.02371 & 0.05099  & 0.02877  & 0.03021 &  0.03372 \\  	
    & & $L_2$ Rel  & \underline{0.03323}  & \textbf{0.03288}   & 0.08347  & 0.05530 & 0.04218 & 0.07513 & 0.05225 & 0.05306  & 0.05914  \\  
    & & $L_\infty$ Rel  & \textbf{0.20686}  & 0.24686   & 0.28441  & 0.24980 & \underline{0.24378}  & 0.38196 & 0.24788 & 0.25710  & 0.29757  \\  
    \bottomrule
    \end{tabular}%
    }
\end{table*}

The experimental results in Table~\ref{table:SWE} show that both versions of our model outperform the baselines on both spatiotemporal interpolation and temporal extrapolation tasks, regardless of whether the dataset is seen or unseen. Specifically, Ours achieves the best performance on spatiotemporal interpolation, while Ours (PINN) performs best on temporal extrapolation. These results support transfer to SWE after fine-tuning and prediction on trajectories generated from initial conditions not used for fine-tuning.


\subsubsection{Cross-PDE Transfer to Compressible Navier-Stokes Equations}
\label{SSS:CNSE}

To further evaluate cross-PDE transfer after fine-tuning, we conduct experiments on the compressible Navier–Stokes equations (CNSE), focusing in particular on learning the time trajectory predicting operator. This experiment is designed not only to evaluate adaptation to a different PDE family but also to assess the potential of our approach in learning an operator that was not included in previous experiments. The dataset consists of 1,000 samples, of which 700 are used for training, 100 for validation, and 200 for testing. The target operator is defined as follows: given the initial condition of four features, the models are required to predict the solution of them at later times $t=0.25, 0.5, 0.75,$ and $1.0$ (see Appendix~\ref{app:datasets_CNSE} for a formal description). Baseline models except DeepONet and DPOT are trained and evaluated using rollout predictions. In contrast, DeepONet and our model directly predict the solutions given the initial condition. Additionally, since DPOT allows the output timestep length to be adjusted, it also predicts the solutions in a single step.

\begin{table*}[ht]
\captionsetup{skip=1.5pt}
\caption{The evaluation results for the operator learning task applied to CNSE. The best performance is marked in \textbf{bold} and the second-best performance is marked with \underline{underline}.}
\label{table:CNSE}
\small
\centering
    \centering
    \captionsetup{labelformat=empty}
    \resizebox{1.0\textwidth}{!}{%
    \begin{tabular}{llccccccccc}
    \toprule
    \textbf{Task}  & \textbf{Metric} & \textbf{Ours} & \textbf{Ours (PINN)}  &  \textbf{DeepONet} & \textbf{FNO} & \textbf{F-FNO} & \textbf{A-FNO} & \textbf{LNO} & \textbf{Poseidon}  & \textbf{DPOT}  \\ 
    \midrule
    \multirow{3}{*}{\shortstack[l]{Operator\\learning}} & $L_1$ Abs  & 0.09497  & 0.09429 & 0.17797  & 0.09581	 & \textbf{0.08734}	 & 0.13637	& 0.10366  & 0.13313  & \underline{0.09164}  \\
    & $L_2$ Rel & \underline{0.54833}  & \textbf{0.54369}  & 1.09567  & 0.57905 & 0.55420 & 0.65798 & 0.62536 & 0.66904  & 0.67869  \\
    & $L_\infty$ Rel & \underline{0.63533}  & \textbf{0.63374}  & 1.09165  & 0.70708 & 0.72876 & 0.85729 & 0.66936 & 0.92909  & 1.15028   \\
    \bottomrule
    \end{tabular}%
    }
\end{table*}

The experimental results in Table~\ref{table:CNSE} show that both versions of our model outperform the baselines on two relative error metrics. Notably, although DPOT’s pre-training corpus already includes the CNSE dataset from PDEBench, our models still achieve superior performance over DPOT on both metrics. While DPOT was originally evaluated on predicting one time step ahead, in our setting the task was extended to predicting up to four time steps ahead. As a result, DPOT exhibited lower performance than reported in the original paper. These results support cross-PDE transfer after fine-tuning for multivariate trajectory prediction.

We further evaluate adaptation to additional PDE datasets and flexibility across input representations in supplementary experiments. First, we evaluate operator learning in the same form as in the CNSE experiments on the Airfoil dataset, which is constructed with an irregular mesh structure. In addition, we evaluate operator learning across different features, rather than within the same features, using the Darcy Flow dataset. Detailed information on these experiments is provided in Appendix~\ref{app:experi}.

\section{Related Works}
\paragraph{In-context learning} Transformers have shown remarkable ICL abilities in various studies. They generalize to unseen tasks by emulating Bayesian predictors \citep{panwar2024incontext} and linear models \citep{zhang2024trained}, while also efficiently performing Bayesian inference through prior-data fitted networks \citep{muller2022transformers}. Their robustness extends to learning different function classes, such as linear and sparse linear functions, decision trees, and two-layer neural networks, even under distribution shifts \citep{garg2022can}. Furthermore, Transformers can adaptively select algorithms based on input sequences, achieving near-optimal performance on tasks like noisy linear models \citep{bai2023transformers}. They are also fast and effective for tabular data classification \citep{hollmann2023tabpfn, hollmann2025tabpfn}. 
Recent work has further connected ICL to richer Bayesian inference procedures \citep{reuter2025can} and implicit optimization dynamics such as gradient descent or Newton-like updates \citep{von2023transformers, fu2024transformers}.

\paragraph{Context-query neural operators.}
GNOT~\cite{hao2023gnot} and UPT~\cite{alkin2024universal} provide related architectures for prediction from flexible inputs and coordinate queries. GNOT supports multiple input functions and irregular meshes, while UPT supports queries of a latent representation at arbitrary spatiotemporal locations. PDE-PFN emphasizes PFN-style episodic prior fitting over PDE instances and context/query sets, enabling the same CDR checkpoint to support multiple prediction settings without task-specific updates. Cross-attention itself is not a unique contribution; a matched GNOT comparison and its attribution limitations are reported in Appendix~\ref{app:gnot_comparison}.

\paragraph{Scientific foundation models} Recent studies have advanced in-context operator learning and PDE solving through Transformer-based models. The work \citet{ye2024pdeformer} introduces PDEformer, a versatile model for solving 1D PDEs with high accuracy in inverse problems. In-context operator learning has also been extended to multi-modal frameworks, as seen in \citet{yang2025fine}, where ICON-LM integrates natural language with equations to outperform traditional models. Additionally, \citet{yang2024pde} and \citet{yang2023context} demonstrate the generalization capabilities of ICON in solving various PDE-related tasks, highlighting ICON's few-shot learning performance across various problems in differential equations.
Several other studies have addressed the problem of solving diverse PDEs using a single trained model \citep{hang2025unisolver,mccabe2023multiple,herde2024poseidon,hao2024dpot}.
However, many of these approaches rely on symbolic PDE information, require task/operator-level demonstrations, or adapt through fine-tuning, making their objectives different from ours.

\section{Conclusion}\label{S:conclusion_limitation}

We presented PDE-PFN, a new method for scientific machine learning to directly approximate the mean of posterior predictive distribution of PDE solutions via in-context learning. Our experiments demonstrate task and coefficient generalization within the CDR family without gradient updates, cross-PDE transfer after fine-tuning, flexible context/query inputs without symbolic equations at inference, and robustness in the evaluated noisy-prior settings. These findings support PFN-style prior fitting as a promising approach to equation-free in-context prediction over PDE solution fields. Future work will examine continuous coefficient priors, broader initial- and boundary-condition distributions, and systematic scaling with model size, prior size and diversity, context size, and distributional distance. Residual and conservation diagnostics, structured distribution shifts, and probabilistic prediction under ambiguous contexts will further assess physical consistency and the scope of generalization.

\clearpage




\bibliography{neurips_2026.bib}
\bibliographystyle{plainnat}

\newpage


\appendix

\section{The Proof of Theorems}\label{app:proof}


\begin{theorem}[Existence of the neural network approximating the posterior predictive distribution]\label{app:neural_sensitivity}
Suppose that for any $ \epsilon > 0 $, there exists a Transformer parameterized by $\hat{\theta}$ such that 
\[
 \mathbb{E}_{c} \left[ KL \left( p_{\hat{\theta}}(\cdot \mid c, D_n), q_n(\cdot \mid c, D_n) \right) \right] < \epsilon,
\] for any realization of  $D_n$.
If the posterior consistency condition Eq.~\eqref{eq:consistency} holds, and for any $ q \in \mathcal{Q} $, $ p_q(c) = p_\pi(c) $ almost everywhere on $ \mathcal{X} $, then the following holds almost surely:
\[
\mathbb{E}_{c} \left[ H \left( p_{\hat{\theta}}(\cdot \mid c, D_n), p_\pi(\cdot \mid c) \right) \right] \xrightarrow{n \to \infty} 0.
\]
\end{theorem}

\begin{proof}
For any two probability distributions $p$ and $q$, recall that the Hellinger distance satisfies
\[
  H(p,q)^2 = \frac{1}{2} \int \Bigl(\sqrt{p(\theta)}-\sqrt{q(\theta)}\Bigr)^2\,d\theta \le \frac{1}{2}\,KL(p\|q).
\]
Thus, if we can show that the Kullback--Leibler divergence $KL\bigl(\hat{p}_{\hat{\theta}}(\cdot\mid c,\tilde{D})\,\|\,p_\pi(\cdot\mid c)\bigr)$ vanishes asymptotically, then the convergence in Hellinger distance follows.
For any $ n, \epsilon $, we derive that
\begin{align}
\mathbb{E}_{c} \left[ H \left( p_{\hat{\theta}}(\cdot \mid c, D_n), p_\pi(\cdot \mid c) \right) \right] & \overset{(1)}{\leq} \mathbb{E}_{c} \left[ H \left( p_{\hat{\theta}}(\cdot \mid c, D_n), q_n(\cdot \mid c, D_n) \right) \right] \\ 
&\quad + \mathbb{E}_{c} \left[ H \left( q_n(\cdot \mid c, D_n), p_\pi(\cdot \mid c) \right) \right] \\
&\overset{(2)}{\leq} \sqrt{\frac{1}{2} \mathbb{E}_{c} \left[ KL \left( p_{\hat{\theta}}(\cdot \mid c, D_n), q_n(\cdot \mid c, D_n) \right) \right]} \\
&\quad + \mathbb{E}_{c} \left[ H \left( q_n(\cdot \mid c, D_n), p_\pi(\cdot \mid c) \right) \right]\\
&\overset{(3)}{\leq} \sqrt{\frac{\epsilon}{2}} + \mathbb{E}_{c} \left[ 1 - \int_{\mathcal{Y}} \sqrt{\int p_q(y \mid c) p_\pi(y \mid c) \mathrm{d}\Pi^n(q)} \mathrm{d}y \right]^{1/2} \\
&\leq \sqrt{\frac{\epsilon}{2}} +\left[ 1 -\int_{\mathcal{C}}p_\pi(c) \int_{\mathcal{Y}} \frac{1}{p_\pi(c)}\sqrt{\int q(y, c) p_\pi(y, c) \mathrm{d}\Pi^n(q)} \mathrm{d}y \mathrm{d}c\right]^{1/2}\\
&\leq \sqrt{\frac{\epsilon}{2}} +  \left[ 1 - \int_{\mathcal{C}}\int_{\mathcal{Y}} \int \sqrt{p_q(y, c) p_\pi(y , c)} \mathrm{d}\Pi^n(q) \mathrm{d}y \mathrm{d}c \right]^{1/2} \\
&= \sqrt{\frac{\epsilon}{2}} + 
\left[ \int  H \left( p_q, p_\pi \right)^2 \mathrm{d}\Pi^n(q) \right]^{1/2} \\
&\leq \sqrt{\frac{\epsilon}{2}} +  \left[\int H \left( p_q, p_\pi\right) d\Pi^n(q) \right]^{1/2} \\
&\overset{(4)}{=} \sqrt{\frac{\epsilon}{2}} +  \left[ \int_{\{ q : H(p_\pi, p_q) > \epsilon \}} H \left( p_q, p_\pi \right) \mathrm{d}\Pi^n(q) \right]^{1/2}\\ 
&\quad + \left[ \int_{\{ q : H(p_\pi, p_q) \leq \epsilon \}} H \left( p_q, p_\pi \right) \mathrm{d}\Pi^n(q) \right]^{1/2}\\
&\overset{(5)}{=} \sqrt{\frac{\epsilon}{2}} + (\Pi^n(\{ q : H(p_\pi, p_q) > \epsilon \}) + \epsilon)^{1/2} \rightarrow \sqrt{\frac{\epsilon}{2}} + \sqrt{\epsilon} \quad \text{a.s.} 
\end{align}

The first inequality ($\overset{(1)}{\leq}$) is derived from the triangle inequality for the Hellinger distance, which states that for any intermediate distribution $ q(\cdot \mid c, D_n) $, we have
\[
H \left( p_{\hat{\theta}}(\cdot \mid c, D_n), p_\pi(\cdot \mid c) \right) \leq H \left( p_{\hat{\theta}}(\cdot \mid c, D_n), q_n(\cdot \mid c, D_n) \right) + H \left( q_n(\cdot \mid c, D_n), p_\pi(\cdot \mid c) \right).
\]
The second inequality ($\overset{(2)}{\leq}$) uses the fact that the Hellinger distance $ H(p, q) $ is bounded above by the square root of the KL divergence $ KL(p \parallel q) $, such that
\[
H(p, q)^2 \leq \frac{1}{2} KL(p \parallel q).
\]
Thus, we can bound the Hellinger distance by the KL divergence. In the third inequality ($\overset{(3)}{\leq}$), we make use of the assumption
\[
\mathbb{E}_{c} \left[ KL \left( p_{\hat{\theta}}(\cdot \mid c, D_n), q_n(\cdot \mid c, D_n) \right) \right] < \epsilon,
\]
and utilize the definition of the Hellinger distance.
In ($\overset{(4)}{=}$), we partition the domain into two regions--one where the Hellinger distance $ H(p_\pi, p_q) $ exceeds $ \epsilon $ and another where it is less than or equal to  $ \epsilon $--and use this partitioning to demonstrate the inequality.

Finally, in ($\overset{(5)}{=}$), by posterior consistency, the region where the Hellinger distance is greater than $ \epsilon $ vanishes as $n \rightarrow \infty$ such that 
\[
\Pi^n \left\{ q : H(p_\pi, p_q) > \epsilon \right\} \to 0 \quad \text{almost surely}.
\]
Since $\epsilon$ is arbitrary, we can conclude that \[
\mathbb{E}_{c} \left[ H \left( p_{\hat{\theta}}(\cdot \mid c, D_n), p_\pi(\cdot \mid c) \right) \right] \xrightarrow{n \to \infty} 0 \quad \text{almost surely}.
\]
\end{proof}

Under its stated assumptions, Theorem~\ref{app:neural_sensitivity} describes approximation of the full posterior predictive distribution as context grows, rather than convergence of the deterministic PDE-PFN implementation. However, PDE-PFN is a deterministic predictor and is trained using the mean squared error loss. Therefore, it does not explicitly model the full posterior predictive distribution or its uncertainty. We now clarify how this deterministic training objective is connected to the posterior predictive mean (PPM): under the squared-loss objective, the population-optimal predictor is precisely the conditional mean of the target solution given the context.

\begin{theorem}[Posterior-predictive-mean inference under squared loss] \label{thm:mse_ppm}
Let $q \in \mathcal{Q}$ denote a latent PDE instance drawn from the prior
distribution $\Pi$. For each $q$, let $u(q): \mathcal{X}\times\mathcal{T}\rightarrow \mathcal{Y}$
denote the corresponding PDE solution field, where for a spatiotemporal
coordinate $c=(x,t)$, the solution value is $y=u(q)(c)$. Let $\widetilde{D}_n=\{(c_i,\tilde{y}_i)\}_{i=1}^{n}$ be an observed context set, where $\tilde{y}_i = u(q)(c_i)+\eta_{c_i}$ and where $\eta_{c_i}$ denotes zero-mean observation noise. We assume that the noise variables are independent across spatiotemporal coordinates and independent of the latent PDE instance $q$.

For a query coordinate $c$, let the training target be $\tilde{Y}(c)=u(q)(c)+\eta_c.$
Consider any deterministic predictor $\tilde{u}_\theta(c,\widetilde{D}_n)$, corresponding to $\tilde{u}_{\theta}(c\mid \widetilde{D}_n)$ in PDE-PFN, trained with the mean squared error loss
\[
    \mathcal{L}(\tilde{u}_\theta) = \mathbb{E}
    \left[ \left\| \tilde{u}_\theta(c,\widetilde{D}_n)-\tilde{Y}(c) \right\|_2^2 \right],
\]
where the expectation is taken over $q\sim\Pi$, the context and query. Then the minimizer satisfies
\[
    \tilde{u}_\theta^{\star}(c,\widetilde{D}_n) = \mathbb{E}\!\left[\tilde{Y}(c)\mid c,\widetilde{D}_n\right]  = \mathbb{E}\!\left[u(q)(c)\mid c,\widetilde{D}_n\right].
\]
\end{theorem}

\begin{proof}
For any measurable deterministic predictor $\tilde{u}_\theta$, we decompose the MSE loss by conditioning on $(c,\widetilde{D}_n)$:
\[
\begin{aligned}
    \mathcal{L}(\tilde{u}_\theta) = \mathbb{E}
    \left[
        \mathbb{E}
        \left[
            \left\| \tilde{u}_\theta(c,\widetilde{D}_n)-\tilde{Y}(c) \right\|_2^2 \,\middle|\, c,\widetilde{D}_n
        \right]
    \right] .
\end{aligned}
\]
Using the standard bias-variance decomposition, the inner expectation becomes
\[
\begin{aligned}
    \mathbb{E}
    \left[
        \left\| \tilde{u}_\theta(c,\widetilde{D}_n)-\tilde{Y}(c) \right\|_2^2 \,\middle|\, c,\widetilde{D}_n
    \right]
    =
    \operatorname{Var}\!\left(\tilde{Y}(c)\mid c,\widetilde{D}_n\right) +
    \left\|
        \tilde{u}_\theta(c,\widetilde{D}_n) - \mathbb{E}\!\left[\tilde{Y}(c)\mid c,\widetilde{D}_n\right]
    \right\|_2^2 .
\end{aligned}
\]
The variance term does not depend on $\tilde{u}_\theta$. Therefore, the loss is minimized
pointwise by
\[
    \tilde{u}_\theta^{\star}(c,\widetilde{D}_n) = \mathbb{E}\!\left[\tilde{Y}(c)\mid c,\widetilde{D}_n\right].
\]
Since the query noise $\eta_c$ is zero-mean and independent of the context noise variables, the latent PDE instance, and the sampled coordinates, we have
\[
    \mathbb{E}[\eta_c\mid c,\widetilde{D}_n]=0.
\]
Therefore,
\[
\begin{aligned}
    \mathbb{E}\!\left[\tilde{Y}(c)\mid c,\widetilde{D}_n\right] =
    \mathbb{E}\!\left[u(q)(c)+\eta_c\mid c,\widetilde{D}_n\right] =
    \mathbb{E}\!\left[u(q)(c)\mid c,\widetilde{D}_n\right].
\end{aligned}
\]
This completes the proof.
\end{proof}

This result does not imply recovery of clean context observations or removal of arbitrary bias or correlated errors in PINN-generated priors.

\begin{corollary}[Effect of context size on the optimal MSE]
\label{cor:context_size_bayes_risk}
Let $\widetilde{D}_n\subset \widetilde{D}_{n+1}$ be nested context sets, and define the
population-optimal MSE with $n$ context observations as
\[
    \mathcal{L}_n^{\star}  =  \inf_{\tilde{u}_\theta} \mathbb{E}
    \left[
        \left\|  \tilde{u}_\theta(c,\widetilde{D}_n)-\tilde{Y}(c)  \right\|_2^2
    \right].
\]
Then,
\[
    \mathcal{L}_{n+1}^{\star} \leq \mathcal{L}_{n}^{\star}.
\]
In other words, the best achievable squared loss is non-increasing as more context observations are provided. Moreover, if the query target is noise-free or if the evaluation is performed against the clean solution $u(q)(c)$, and the latent PDE instance $q$ becomes identifiable from increasingly many context observations, then the corresponding clean-solution prediction error can approach zero as $n$ increases.
\end{corollary}

\begin{proof}
By Theorem~\ref{thm:mse_ppm}, the population-optimal predictor with $n$
context observations is
\[
    \tilde{u}^{\star}_{\theta,n}(c,\widetilde{D}_n) = \mathbb{E}\!\left[\tilde{Y}(c)\mid c,\widetilde{D}_n\right],
\]
and the corresponding optimal MSE is
\[
    \mathcal{L}_n^{\star} = \mathbb{E}
    \left[
        \operatorname{Var}
        \left(  \tilde{Y}(c)\mid c,\widetilde{D}_n  \right)
    \right].
\]
Since $\widetilde{D}_n\subset \widetilde{D}_{n+1}$, conditioning on $\widetilde{D}_{n+1}$ provides at least as much information as conditioning on $\widetilde{D}_n$. By the law of total variance,
\[
    \mathbb{E}
    \left[
        \operatorname{Var}
        \left(  \tilde{Y}(c)\mid c,\widetilde{D}_n   \right)
    \right] =
    \mathbb{E}
    \left[
        \operatorname{Var}
        \left(  \tilde{Y}(c)\mid c,\widetilde{D}_{n+1}  \right)
    \right] 
    +
    \mathbb{E}
    \left[
        \operatorname{Var}
        \left(  \mathbb{E}[\tilde{Y}(c)\mid c,\widetilde{D}_{n+1}] \mid c,\widetilde{D}_n  \right)
    \right].
\]
The second term is non-negative, and therefore
\[
    \mathcal{L}_{n+1}^{\star} \leq \mathcal{L}_{n}^{\star}.
\]
This shows that increasing the number of context observations cannot increase the population-optimal squared loss.
\end{proof}


\section{Datasets} \label{app:datasets}

\subsection{Family of 2D Convection-Diffusion-Reaction Equations} \label{app:datasets_CDR}
The formulation of the family of 2D CDR equations with three reaction terms, used in Section~\ref{SS:CDR}, is given as follows.  
\begin{equation}
\begin{aligned}
\textnormal{2D CDR: }  \ u_t + \boldsymbol{\beta} \cdot \nabla u - \nabla \cdot (\boldsymbol{\nu} \nabla u) - \sum_{j=1}^{3} \rho_{j} f_j(u) = 0,  \ x \in \mathcal{X}, \ t \in \mathcal{T}, \\
\end{aligned}
\end{equation}  
where $\mathcal{X} = [0, \pi] \times [0, \pi]$ and  $\mathcal{T} = [0, 1]$ are solution domain with a resolution of $32 \times 32 \times 21$. For the family of 2D CDR equations, we use analytic solutions whenever they are available. When an analytic solution is not available, the solution is generated using the spectral method, following the same approach as in \citep{cho2023hypernetwork}. The solutions are computed over the time interval $t \in [0.0, 1.0]$ with 1001 time steps at intervals of 0.001, from which 21 time steps at intervals of 0.05 (from $t=0.0$ to $t=1.0$) are used in the experiments. We pre-train our model on the family of 2D CDR equations with coefficients ($\beta_x, \beta_y, \nu_x, \nu_y, \rho_1, \rho_2, \rho_3 \in {0.0, 0.5, 1.0}$). This results in a total of 2,187 unique PDEs. For PINN prior, because a PINN can be trained on only a single PDE at a time, we train an individual PINN for each of the 2,187 distinct 2D CDR equations.

\subsection{2D Shallow Water Equation}  \label{app:datasets_SWE}
The 2D shallow water equation (SWE) dataset used in Section~\ref{SSS:SWE} is obtained from PDEBench~\citep{takamoto2022pdebench}\footnote{Dataset source: https://github.com/pdebench/PDEBench}. The SWE is widely used as a benchmark dataset and corresponds to a hyperbolic PDE of the following form:
\begin{equation}
\begin{aligned}
\partial_t h + \partial_x (hu) + \partial_y (hv) = 0, \\ 
\partial_t (hu) + \partial_x \left(u^2h + \frac{1}{2} g_r h^2 \right) + \partial_y (uvh) = -g_r h \partial_x b, \\ 
\partial_t (hv) + \partial_y \left(v^2h + \frac{1}{2} g_r h^2 \right) + \partial_x (uvh) = -g_r h \partial_y b,
\end{aligned}
\end{equation} 
where $u$ and $v$ denote the horizontal and vertical velocities, $h$ is the water depth, and $g_r$ is the gravitational acceleration. The SWE is derived from the general Navier–Stokes equations and is well suited for modeling free-surface flow problems.

The dataset contains simulation results for the water height $h$ and is generated using the PyClaw~\citep{ketcheson2012pyclaw} Python package. It consists of samples produced from 1,000 different initial conditions. The simulation domain is $\mathcal{X} = [-2.5, 2.5]^2$ over the time interval $\mathcal{T} = [0.0, 1.0]$, with a resolution of $128 \times 128 \times 101$. From this, we subsample 21 time steps at intervals of 0.05 for use in our experiments.

\subsection{2D Compressible Navier-Stokes Equation}  \label{app:datasets_CNSE}

The 2D compressible Navier–Stokes equations dataset, referred to as CNSE, used in Section~\ref{SSS:CNSE} was obtained from PDEBench. The Navier–Stokes equation is one of the most widely adopted benchmarks for evaluating SciML PDE solvers, and in this work we focus on the compressible fluid case. The governing equations are given by
\begin{equation}\label{eq:CNSE}
\begin{aligned}
\quad \partial_t \rho + \nabla \cdot (\rho \mathbf{v}) &= 0, \\
\quad \rho \left( \partial_t \mathbf{v} + \mathbf{v} \cdot \nabla \mathbf{v} \right) &= -\nabla p + \eta \Delta \mathbf{v} + \left( \zeta + \frac{\eta}{3} \right) \nabla (\nabla \cdot \mathbf{v}), \\
\quad \partial_t \left( \epsilon + \frac{1}{2} \rho v^2 \right)& + \nabla \cdot \left[ \left( \epsilon + p + \frac{1}{2} \rho v^2 \right) \mathbf{v} - \mathbf{v} \cdot \boldsymbol{\sigma}' \right] = 0,
\end{aligned}
\end{equation}
where $\rho$ denotes the mass density, $\mathbf{v}$ the velocity, $p$ the gas pressure, $\epsilon = p / (\Gamma - 1)$ the internal energy with $\Gamma = 5/3$, $\boldsymbol{\sigma}'$ the viscous stress tensor, and $\eta, \zeta$ the shear and bulk viscosity, respectively.

Numerical solutions were computed in the spatial domain $ \mathcal{X} = [0.0, 1.0]^2$ and temporal domain $\mathcal{T} =[0.0, 1.0]$ using a second-order HLLC scheme~\citep{toro1994restoration} with the MUSCL method~\citep{van1979towards} for the inviscid part, and a central difference scheme for the viscous part. We use the solution computed by setting the Mach number as $M = |\mathbf{v}| / c_s = 1.0$, where $c_s = \Gamma p / \rho$ is the sound velocity, and both viscosity coefficients as $\eta = \zeta = 0.1$. Periodic boundary conditions were applied, and distinct initial conditions were generated from random fields. The dataset contains simulation results for $x$- and $y$-velocities $\mathbf{v}$, pressure $p$, and density $\rho$. The dataset consists of 1,000 samples simulated from different initial conditions, originally computed on a $128 \times 128 \times 21$ grid. Each feature in the dataset was normalized individually for the experiment. For our experiments, we employed a subsampled version of the dataset with resolution $64 \times 64 \times 5$. In the experiment we consider a field $ \mathbf{u}(x,t)=[\mathbf{v}(x,t), p(x,t), \rho(x,t)] \in\mathbb{R}^4$ for $x \in \mathcal{X}, \; t \in \mathcal{T}$. Here $\mathbf{u}(x,t)$ denotes a four-dimensional vector field representing four features of the system and the initial condition $\mathbf{u}(x,0.0) = \mathbf{u}_0(x)$ for $x \in \mathcal{X}$. We are interested in learning the trajectory-predicting operator $\mathcal{G}^\dagger : \mathbf{u}_0(x) \;\mapsto\; \big(\mathbf{u}(x,0.25), \, \mathbf{u}(x,0.5), \, \mathbf{u}(x,0.75), \, \mathbf{u}(x,1.0)\big),$ which maps the initial state $\mathbf{u}_0 \in L^2\big((0.0,1.0)^2;\mathbb{R}^4\big)$ to the solution at four future time instances.

\subsection{Airfoil}  \label{app:datasets_airfoil}

The Airfoil dataset used in Section~\ref{app:experi_Airfoil} was obtained from the dataset provided by~\citep{li2023transformer}\footnote{Dataset Source: https://github.com/BaratiLab/OFormer}. This dataset was originally generated in~\citep{pfaff2020learning} by solving the Euler equations for compressible flow using the finite volume method built into the SU2 library~\citep{palacios2013stanford}\footnote{Dataset and numerical solver source: https://github.com/merantix-momentum/gnn-bvp-solver}. The Euler equations have the following formulation:

\begin{equation}\label{eq:airfoil}
\begin{aligned}
  \partial_t \rho + \nabla \cdot (\rho \mathbf{v}) &= f_1, \\
\partial_t (\rho \mathbf{v}) + \nabla \cdot (\rho \mathbf{v} \otimes \mathbf{v} + p\mathbb{I}) &= \mathbf{f}_2, \\
\partial_t (\rho E) + \nabla \cdot (\rho E \mathbf{v} + p\mathbf{v}) &= f_3, \\
\rho := \rho(x,t), \ \mathbf{v} := \mathbf{v}(x,t), \ p := p(x,t), \ &x \in \mathcal{X}, \ t \in [0,T],
\end{aligned}
\end{equation}

where $\rho$ denotes the density, $\mathbf{v}$ the velocity field, $p$ the pressure, $E$ the total energy per unit mass, and $f_1,\mathbf{f}_2,f_3$ are generic source terms. \citet{li2023transformer} formulated this problem as a non-Markovian initial value problem and used a significantly larger time step size. The dataset spans $t \in \mathcal{T} = [0.0,4.8]$ with 101 time steps, each spaced by a time interval of 0.24. At each timepoint, the dataset contains values for 5,233 irregular mesh nodes, including node position, node type, velocity, pressure, and density. Similar to CNSE, we consider a field $ \mathbf{u}(x,t)=[\mathbf{v}(x,t), p(x,t), \rho(x,t)] \in\mathbb{R}^4$ for $x \in \mathcal{X}, \; t \in \mathcal{T}$. Here $\mathbf{u}(x,t)$ denotes a four-dimensional vector field representing four features of the system. The initial condition is given by $\mathbf{u}(x,0) = \mathbf{u}_0(x)$ for $x \in \Omega$. We want to learn the trajectory-predicting operator $\mathcal{G}^\dagger : \mathbf{u}_0(x) \;\mapsto\; \big(\mathbf{u}(x,1.2), \, \mathbf{u}(x,2.4), \, \mathbf{u}(x,3.6), \, \mathbf{u}(x,4.8)\big),$ which maps the initial state $\mathbf{u}_0 \in L^2\big(\mathcal{X};\mathbb{R}^4\big)$ to the solution at four future time instances.

\subsection{Darcy Flow}    \label{app:datasets_Darcy}
The Darcy Flow dataset used in Section~\ref{app:experi_Darcy} is also obtained from PDEBench. Darcy Flow is a widely adopted benchmark for evaluating operator learning models, formulated as a time-independent elliptic PDE. The governing equation is given by:
\begin{equation}\label{eq:darcy}
\begin{aligned}
  -\nabla (a(x) \nabla u(x)) &= f(x), \quad x \in \mathcal{X}, \\
  u(x) &= 0, \quad x \in \partial \mathcal{X},
\end{aligned}
\end{equation}
where the source term $f(x)$ is fixed to a constant value of $0.1$. The dataset consists of steady-state solutions $u$ corresponding to the viscosity fields $a$. To obtain these solutions, we simulate the following time-dependent formulation with random initial conditions until convergence to steady state:
\begin{equation}
    \partial_t u(x, t) - \nabla \left( a(x) \nabla u(x, t) \right) = f(x), \quad x \in \mathcal{X}
    \label{eq:darcy2}
\end{equation}
Numerical solutions are simulated using a second-order central difference scheme in both time and space. Each data sample is represented as an input–output pair, where the input is a viscosity field $a(x)$ and the output is the corresponding steady-state solution $u(x)$. Both inputs and outputs are discretized on a $128 \times 128$ spatial grid. The dataset contains 2,000 samples, each generated from distinct random initial conditions within the simulation domain $\mathcal{X} = [0.0, 1.0]^2$. Unlike the CNSE setting, here we aim to learn the operator mapping a viscosity field $a(x) \in L^2\big((0.0,1.0)^2;\mathbb{R}\big)$ to the corresponding steady-state solution $u(x)$ for $x \in \mathcal{X}$. Specifically, we want to learn the operator $\mathcal{G}^\dagger : a(x) \;\mapsto\; u(x)$.

\subsection{3D Diffusion-Reaction Equations} \label{app:datasets_3D_DR}
The formulation of the 3D diffusion-reaction equations, used in Section~\ref{app:experi_3d_gray_scott}, is given as follows.  
\begin{equation}
\begin{aligned}
u_t - \nabla \cdot (\boldsymbol{\nu} \nabla u) - \rho u\cdot(1-u) = 0,  \ x \in \mathcal{X}, \ t \in \mathcal{T}, \\
\end{aligned}
\end{equation}  
where $\mathcal{X} = [0, \pi] \times [0, \pi] \times [0, \pi]$ and  $\mathcal{T} = [0, 1]$ are solution domain with a resolution of $16 \times 16 \times 16 \times 21$. For generations, we used the same method and process with the family of 2D CDR equations. We conduct experiment on coefficients ($\nu_x, \nu_y, \nu_z, \rho \in {0.0, 0.5, 1.0, 1.5}$). This results in a total of 256 unique PDEs.

\subsection{3D Gray-Scott}    \label{app:datasets_gray_scott}
The 3D Gray-Scott equations dataset used in Appendix~\ref{app:experi_3d_gray_scott} is obtained from APEBench~\citep{koehler2024apebench}, a benchmark suite for autoregressive neural emulators of PDEs. In APEBench, 3D Gray-Scott is categorized as a reaction-diffusion system and is provided as a multi-channel PDE. The governing equations are given by
\begin{equation}
\begin{aligned}
\frac{\partial u_0}{\partial t} &= \nu_0 \nabla \cdot \nabla u_0 - u_0 u_1^2 + f(1-u_0), \\
\frac{\partial u_1}{\partial t} &= \nu_1 \nabla \cdot \nabla u_1 + u_0 u_1^2 - (f+k)u_1,
\end{aligned}
\end{equation}
where $u_0$ and $u_1$ denote the two chemical species, $\nu_0$ and $\nu_1$ are the diffusion coefficients, and $f$ and $k$ are the feed and kill rates, respectively. APEBench includes Gray-Scott among its supported semi-linear reaction-diffusion equations and uses a Fourier pseudo-spectral ETDRK solver for procedural trajectory generation. Following the APEBench setup, the Gray-Scott dynamics are generated in the physical parameterization. In particular, the default Gray-Scott configuration uses diffusion coefficients $\nu_0 = 2 \times 10^{-5}$ and $\nu_1 = 1 \times 10^{-5}$. For the Gray-Scott pattern type considered in APEBench, the feed and kill rates are selected according to the predefined pattern regimes introduced by Pearson~\cite{pearson1993complex}. Among those regimes, we only consider the case with $f=0.04$ and $k=0.06$. APEBench further notes that the 3D Gray-Scott experiments use a Gaussian-blob-type initialization for each channel. Both inputs and outputs are discretized on a $32 \times 32 \times 32$ spatial grid, and the dataset spans $t \in T = [0.0, 1000.0]$ with $51$ time steps for training and $t \in T = [0.0, 4000.0]$ with $201$ time steps for validation and testing, each spaced by a time interval of $20$. We subsample the dataset with five time intervals, i.e., $\Delta t=100$.

For our experiments, we use the 3D Gray-Scott dataset to evaluate operator learning in a higher-dimensional complex system. Let $u(x,t) = [u_0(x,t), u_1(x,t)] \in \mathbb{R}^2$ for $x \in \mathcal{X}$ and $t \in \mathcal{T}$ denote the two-channel state variable. Given the state at a current time step, the task is to predict the system state 5 time steps ahead. Thus, we learn the time-advancing operator $\mathcal{G}^{\dagger}: u(x,t) \mapsto u(x,t+\Delta t),$ which maps the current two-species concentration field to the future state after five time steps.

The specific information about PDE types used in each experiment is described in Table~\ref{table:Dataset}. Moreover, the detailed values regarding how many training, validation, and testing data points in each experiment are provided in Appendix~\ref{app:experimental_details}. For the spatiotemporal interpolation with unseen coefficients in the family of 2D CDR equation experiments, we use the same time range as in the corresponding task with seen coefficients.

\begin{table*}[htbp]
\caption{Time range of data points used for training, validation, and testing in all experiments. \label{table:Dataset}}
\small
\centering
\resizebox{1.0\linewidth}{!}{%
\begin{tabular}{llccccc}
\toprule
\multirow{2}{*}{\textbf{\begin{tabular}[l]{@{}l@{}} \textbf{PDE} \\ \textbf{(Section)} \end{tabular}}} & \multirow{2}{*}{\textbf{Task}} & \multirow{2}{*}{\textbf{Train}} & \multicolumn{2}{c}{\textbf{Valid}}    &  \multicolumn{2}{c}{\textbf{Test}}\\ 
\cmidrule(lr){4-5} \cmidrule(lr){6-7}
 &  &  & \textbf{Context} & \textbf{Query} &  \textbf{Context} & \textbf{Query} \\ 
\midrule
\multirow{7}{*}{\begin{tabular}[l]{@{}l@{}} Family of 2D \\ CDR equations\\ (Section~\ref{SS:CDR}) \end{tabular}} & \begin{tabular}[l]{@{}l@{}}Spatiotemporal \\ interpolation\end{tabular}   &    $t \in [0.0, 0.5]$    &   \begin{tabular}[c]{@{}c@{}}predefined in \\ $t \in [0.0, 0.5]$\end{tabular}   &   \begin{tabular}[c]{@{}c@{}}grid at \\ $t=0.25$\end{tabular}      &    \begin{tabular}[c]{@{}c@{}}predefined in \\ $t \in [0.0, 0.5]$\end{tabular}   &    \begin{tabular}[c]{@{}c@{}}grids at \\ $t \in \{0.05, 0.15, 0.35, 0.45\}$\end{tabular} \\
\cmidrule{2-7}
& \begin{tabular}[l]{@{}l@{}}Temporal \\ extrapolation\end{tabular} &    $t \in [0.0, 0.5]$    &   \begin{tabular}[c]{@{}c@{}} grid at \\ $t=0.5$\end{tabular}   &    \begin{tabular}[c]{@{}c@{}} grid at \\ $t=0.6$\end{tabular}    &    \begin{tabular}[c]{@{}c@{}} grid at  \\   $t=0.6$  \end{tabular}    &    \begin{tabular}[c]{@{}c@{}}  grids at  \\  $t \in [0.6, 1.0]$\end{tabular}  \\
\cmidrule{2-7}
& \begin{tabular}[l]{@{}l@{}} Unseen coeff \\ temporal \\ extrapolation\end{tabular} & - & - & - & \begin{tabular}[c]{@{}c@{}} grid at \\ $t=0.1$\end{tabular} & \begin{tabular}[c]{@{}c@{}} grids at \\ $t \in [0.2, 0.5]$\end{tabular} \\
\midrule
\multirow{4}{*}{\begin{tabular}[l]{@{}l@{}} 2D SWE \\ (Section~\ref{SSS:SWE}) \end{tabular}} & \begin{tabular}[l]{@{}l@{}}Spatiotemporal \\ interpolation\end{tabular} & $t \in [0.0, 0.5]$    &   \begin{tabular}[c]{@{}c@{}}predefined in \\ $t \in [0.0, 0.5]$\end{tabular}   &   \begin{tabular}[c]{@{}c@{}}grid at \\ $t=0.25$\end{tabular}      &    \begin{tabular}[c]{@{}c@{}}predefined in \\ $t \in [0.0, 0.5]$\end{tabular}   &    \begin{tabular}[c]{@{}c@{}}grids at \\ $t \in \{0.05, 0.15, 0.35, 0.45\}$\end{tabular} \\ 
\cmidrule{2-7}
& \begin{tabular}[l]{@{}l@{}}Temporal \\ extrapolation\end{tabular}  &  $t \in [0.0, 0.5]$    &   \begin{tabular}[c]{@{}c@{}} grid at \\ $t=0.5$\end{tabular}   &    \begin{tabular}[c]{@{}c@{}} grid at \\ $t=0.6$\end{tabular}    &    \begin{tabular}[c]{@{}c@{}} grid at  \\   $t=0.6$  \end{tabular}    &    \begin{tabular}[c]{@{}c@{}}  grids at  \\  $t \in [0.6, 1.0]$ \end{tabular} \\
\midrule
\begin{tabular}[l]{@{}l@{}} 2D CNSE \\ (Section~\ref{SSS:CNSE}) \end{tabular} & Operator learning &  \begin{tabular}[c]{@{}c@{}} $t \in \{0.0, 0.25,$ \\ $0.5, 0.75, 1.0\}$ \end{tabular} & \begin{tabular}[c]{@{}c@{}} grid at  \\   $t=0.0$  \end{tabular} & \begin{tabular}[c]{@{}c@{}} grids at  \\   $t \in \{0.25, 0.5, 0.75, 1.0\}$  \end{tabular} & \begin{tabular}[c]{@{}c@{}} grid at  \\   $t=0.0$  \end{tabular} & \begin{tabular}[c]{@{}c@{}} grids at  \\   $t \in \{0.25, 0.5, 0.75, 1.0\}$  \end{tabular} \\ 
\midrule
\begin{tabular}[l]{@{}l@{}} Airfoil \\ (Section~\ref{app:experi_Airfoil}) \end{tabular} & Operator learning &  \begin{tabular}[c]{@{}c@{}} $t \in \{0.0, 1.2,$ \\ $2.4, 3.6, 4.8\}$ \end{tabular} & \begin{tabular}[c]{@{}c@{}} grid at  \\   $t=0.0$  \end{tabular} & \begin{tabular}[c]{@{}c@{}} grids at  \\   $t \in \{1.2, 2.4, 3.6, 4.8\}$  \end{tabular} & \begin{tabular}[c]{@{}c@{}} grid at  \\   $t=0.0$  \end{tabular} & \begin{tabular}[c]{@{}c@{}} grids at  \\   $t \in \{1.2, 2.4, 3.6, 4.8\}$  \end{tabular} \\ 
\midrule
\begin{tabular}[l]{@{}l@{}}Darcy Flow \\ (Section~\ref{app:experi_Darcy}) \end{tabular} & Operator learning & \multicolumn{5}{c}{$a \mapsto u$}  \\ 
\midrule
\multirow{4}{*}{\begin{tabular}[l]{@{}l@{}} 3D diffusion-\\ reaction equations\\ (Section~\ref{app:experi_3d_dr}) \end{tabular}} & \begin{tabular}[l]{@{}l@{}}Spatiotemporal \\ interpolation\end{tabular}   &    $t \in [0.0, 0.5]$    &   \begin{tabular}[c]{@{}c@{}}predefined in \\ $t \in [0.0, 0.5]$\end{tabular}   &   \begin{tabular}[c]{@{}c@{}}grid at \\ $t=0.25$\end{tabular}      &    \begin{tabular}[c]{@{}c@{}}predefined in \\ $t \in [0.0, 0.5]$\end{tabular}   &    \begin{tabular}[c]{@{}c@{}}grids at \\ $t \in \{0.05, 0.15, 0.35, 0.45\}$\end{tabular} \\
\cmidrule{2-7}
& \begin{tabular}[l]{@{}l@{}}Temporal \\ extrapolation\end{tabular} &    $t \in [0.0, 0.5]$    &   \begin{tabular}[c]{@{}c@{}} grid at \\ $t=0.5$\end{tabular}   &    \begin{tabular}[c]{@{}c@{}} grid at \\ $t=0.6$\end{tabular}    &    \begin{tabular}[c]{@{}c@{}} grid at  \\   $t=0.6$  \end{tabular}    &    \begin{tabular}[c]{@{}c@{}}  grids at  \\  $t \in [0.6, 1.0]$\end{tabular}  \\
\midrule
\begin{tabular}[l]{@{}l@{}}3D Gray-Scott \\ (Section~\ref{app:experi_3d_gray_scott}) \end{tabular} & Operator learning &  \begin{tabular}[c]{@{}c@{}} $t \in \{0, 100,$ \\ $ \dots, 1000\}$ \end{tabular} & \begin{tabular}[c]{@{}c@{}} grid at  \\   $t \in \{1000, \dots, 3900\}$  \end{tabular} & \begin{tabular}[c]{@{}c@{}} grids at  \\   $t \in \{1100, \dots, 4000\}$  \end{tabular} & \begin{tabular}[c]{@{}c@{}} grid at  \\    $t \in \{1000, \dots, 3900\}$  \end{tabular} & \begin{tabular}[c]{@{}c@{}} grids at  \\   $t \in \{1100, \dots, 4000\}$  \end{tabular} \\ 

\bottomrule
\end{tabular}%
}
\end{table*}


\section{PINN Used in Prior Generation}
\label{app:PINN}

In this study, we utilize the PINN introduced by \citet{raissi2019physics} to generate PINN priors. The loss function employed during the training of the PINN is as follows:
\begin{equation}
\mathcal{L} = \mathcal{L}_u + \mathcal{L}_f + \mathcal{L}_b,
\end{equation}
where $\mathcal{L}_u, \mathcal{L}_f$ and $\mathcal{L}_b$ are defined as
\begin{equation}
\mathcal{L}_u = \frac{1}{N_u} \sum \left( \Tilde{u}(x, 0) - u(x, 0) \right)^2, \;\; \mathcal{L}_f = \frac{1}{N_f} \sum \left( \mathcal{N}(t, x, u, \alpha) \right)^2, \;\; \mathcal{L}_b = \frac{1}{N_b} \sum \left( \Tilde{u}(0, t) - \Tilde{u}(2 \pi, t) \right)^2,
\end{equation}
for $N_u$ points at initial condition, $N_f$ collocation points, and $N_b$ boundary points.

The generation time for the PINN prior varies depending on the number of training epochs and the PINN loss threshold. Based on the configuration used in our experiments, it took approximately 373 seconds to generate the prior for a single coefficient combination of the 2D CDR and utilized up to 4,552 MB of memory.


\section{Detailed Description of the Baselines} \label{app:Baselines}

We provided descriptions of each baseline discussed in Section~\ref{SS:Experimental_setup}. 
\begin{itemize}
    \item DeepONet is a neural operator architecture designed to learn operators mapping input functions to output functions. It combines branch and trunk networks to predict values in a function space.
    \item FNO learns solution operators for partial differential equations (PDEs) using the Fourier transform. By mapping inputs to a frequency domain, FNO captures complex patterns and long-range dependencies and models complex systems.
    \item F-FNO extends the Fourier Neural Operator by factorizing its layers to reduce computational costs. This factorization enables efficient learning of solution operators for complex systems.
    \item A-FNO is a variant of the Fourier Neural Operator that dynamically adjusts the resolution of the frequency domain during training. This adaptation aims to capture relevant features across scales, enabling more flexible modeling of complex systems.
    \item Poseidon is a multiscale operator transformer architecture designed to model solution operators for PDEs with continuous-in-time capability. It introduces time-conditioned layer normalization and processes multi-resolution input features to capture complex spatiotemporal dynamics.
    \item DPOT is a scalable neural operator model designed for PDEs, featuring a Fourier attention mechanism that enables efficient learning of complex spatiotemporal dynamics. Its architecture incorporates an auto-regressive denoising pre-training strategy, facilitating stable and flexible modeling across diverse PDE datasets with varying resolutions and geometries.
\end{itemize} 
In the spatiotemporal interpolation task, models other than DeepONet are excluded from comparison because they cannot process mesh-structured inputs. DeepONet requires input data from fixed positions for each dataset; therefore, for a fair comparison, our model is also provided with context data from the same positions during evaluation. For Poseidon and DPOT, which are scientific foundation models providing pre-trained weights, we utilized their released weights and performed fine-tuning. For a fair comparison, we trained them and our model using only half the number of epochs compared to the other models. Furthermore, computational costs, memory usage, experiment settings, and hyperparameters are comprehensively outlined in Appendix~\ref{app:experimental_details}.
\begin{table*}[ht]
\captionsetup{skip=1.5pt}
\caption{Major comparisons between baselines and our model on the shape of input and target data (context and query in our model). Neural operator baselines (excluding DeepONet) can only predict targets at predetermined locations and are limited to processing grid-structured input data. DeepONet, on the other hand, can handle inputs of various shapes but requires fixed-coordinate input values, though it can produce target solutions at desired coordinates. Our model accepts variable-sized coordinate–value context sets and supports coordinate queries, although accuracy depends on prior coverage and context informativeness.} 
\label{table:baseline_comparison}
    \begin{center}
    \begin{small}
    \makebox[1.0\textwidth][c]{
    \resizebox{1.0\textwidth}{!}{%
    \begin{tabular}{lcccc}
        \toprule
        \textbf{Data}   & \begin{tabular}[c]{@{}c@{}}\textbf{FNO \& F-FNO \& A-FNO }\\ \textbf{\& Poseidon \& DPOT}\end{tabular} & \textbf{LNO} & \textbf{DeepONet} & \textbf{Ours}\\
        \midrule
        Input data   & Predefined grid  &  Predefined mesh  &  Predefined mesh &  \textbf{Arbitrary Mesh} \\
        Target data  & Predefined grid  &  Predefined mesh  & \textbf{Arbitrary Mesh} & \textbf{Arbitrary Mesh}\\
        \bottomrule
    \end{tabular}}}
    \end{small}
    \end{center}
\end{table*}


\section{Model Architecture} \label{app:model_architecture}

\begin{figure}[ht]
  \centering
  \includegraphics[width=0.9\textwidth]{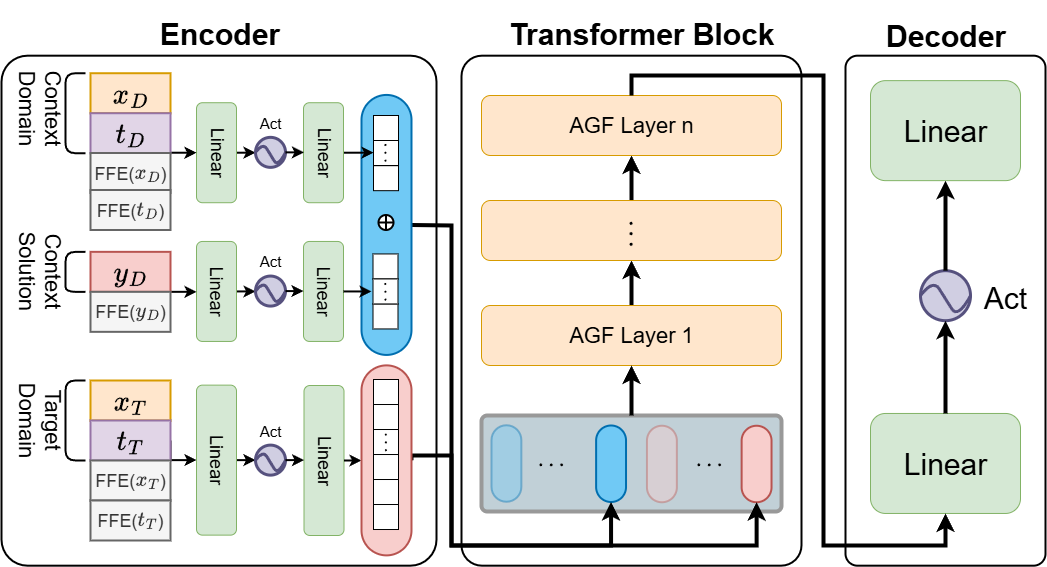}
  \caption{A diagram of PDE-PFN's architecture. The diagram is presented as an extension of Figure~\ref{fig:Figure2}. The model consists of three main components: an encoder, a Transformer block, and a decoder. FFE denotes the Fourier feature embedding, and Act denotes the activation function.}
\end{figure}

Our model is fundamentally based on the Prior-Fitted Network (PFN) architecture~\citep{muller2022transformers}, and consists of three main components: an encoder, a Transformer block, and a decoder. In this section, we describe the encoder. For clarity, we denote the training contexts and queries as $D$ and $T$, respectively, consistent with the main text. Each input in $D$ and $T$ is composed of spatial coordinates $x$, temporal coordinates $t$, and solution values $y$; for queries $T$, the solution values are masked. The encoder optionally augments the raw inputs with Fourier features before MLP processing; FFE dimension zero retains only the raw inputs.

\paragraph{Encoder}
The role of FFE is to augment the raw inputs with high-frequency components, enabling the model to better capture complex solution patterns. This advantage of using FFE in conjunction with MLPs has been verified in \citep{tancik2020fourier,ma2025robustifying}. For grid-based inputs, we can employ a discrete Fourier transform via fast Fourier transform (FFT). For coordinate-based inputs, we use sinusoidal feature mappings with predetermined frequencies. Specifically, given an input $v \in \mathbb{R}$, we define the embedding as
\begin{equation}
\text{FFE}(v; n)=[sin(v\cdot \omega_1), cos(v\cdot \omega_1), \cdots ,sin(v\cdot \omega_n),cos(v\cdot \omega_n)],
\end{equation}

where the frequency coefficients are set as $\omega_i = 2\pi / i$. After applying FFE, the encoder constructs three types of feature vectors:
$l_{CD} := [x_D, t_D, \text{FFE}(x_D), \text{FFE}(t_D)]$ for context domain information,
$l_{CS} := [y_D, \text{FFE}(y_D)]$ for context solutions, and
$l_{QD} := [x_T, t_T, \text{FFE}(x_T), \text{FFE}(t_T)]$ for query domain information, by concatenating the raw inputs and FFE results. Each of these vectors is passed through a separate MLP to learn PDE-specific embeddings. Each MLP consists of two linear layers and a single activation function between linear layers. For the activation, we employ a rational activation function~\citep{molina2019pade}, which has an adaptive nature and has been shown to provide greater flexibility compared to conventional nonlinearities. Finally, $l_{CD}$ and $l_{CS}$ are concatenated after passing each MLP to form the context representation used by the Transformer block.

\paragraph{Transformer block}
The encoder output is then processed through a Transformer block, which consists of a sequence of Transformer layers. Instead of the \textit{vanilla} Transformer layer, we adopt the attentive graph filter (AGF) layer~\cite{wi2025learning}, which provides greater flexibility while reducing computational complexity. The AGF layer enhances standard linear attention by learning attention in the singular value domain. Specifically, it decomposes the attention operation into a set of spectral components and adaptively learns their importance, enabling the model to capture the dependencies more effectively. 
Specifically, the attention matrix $A \in \mathbb{R}^{n \times m}$ can be decomposed using singular value decomposition (SVD) as $A=U \Sigma V^T$, where $U$ and $V$ are orthogonal matrices, and $\Sigma = \text{diag}(\sigma_1, \sigma_2, \ldots, \sigma_r)$ contains the singular values. Therefore, as discussed in Section~\ref{S:Method}, when using the AGF layer, we include a regularization term to ensure the orthogonality of the matrices $U$ and $V$. The regularization term is defined as follows:
$$
\mathcal{L}_{\text{AGF}} = \frac{1}{n^2} \left( \left\| U^\top U  - \mathbf{I} \right\| \right) + \frac{1}{m^2} \left( \left\| V V^\top - \mathbf{I} \right\| \right),
$$
where $\mathbf{I}$ is identity matrix. AGF learns to adaptively reweight these spectral components, effectively filtering the most informative modes while suppressing noise or redundant information.
By learning attention in the singular value domain, AGF provides a more powerful representation of global dependencies compared to conventional linear attention, while avoiding the quadratic complexity of the \textit{vanilla} Transformers.

To further enrich the model capacity, we separate the parameters used for self-attention and cross-attention within each layer. As illustrated in Figure~\ref{fig:Figure2}, self-attention is applied among the contexts $D$, while cross-attention allows queries $T$ to attend to contexts $D$. In the figure, self-attention and cross-attention are represented by blue and red arrows, respectively. The numbers of blue rods $D$ and red rods $T$ are described in Table~\ref{table:data_num} in Appendix~\ref{app:experimental_details}.

\paragraph{Decoder}
The output of the Transformer block is finally passed through a decoder, implemented as a simple MLP. Similar to the encoder, this MLP consists of two linear layers and a single activation function between linear layers. The decoder maps the latent representations produced by the Transformer into the predicted solution values at the queried coordinates.


\section{Training}
\label{app:Training Algorithm}

Unlike other task-specific models, our model is trained in an integrated manner, enabling it to solve multiple tasks simultaneously in that dataset. For cases where both spatiotemporal interpolation and temporal extrapolation tasks must be evaluated, such as in the family of 2D CDR equations and SWE datasets, we perform training with temporal extrapolation as the primary objective. In this setup, other baselines treat grid data points at $t_n$ and $t_{n+1}$ as fixed input–output pairs. In contrast, our model constructs pairs by randomly mixing the data points within each $(t_n, t_{n+1})$ pair independently, rather than preserving their exact correspondence. This strategy enables the model to learn in a way that supports both spatiotemporal interpolation and temporal extrapolation within a unified framework. For the CNSE, Airfoil, and Darcy Flow datasets, where we conducted experiments only on operator learning, random mixing was not applied.

\paragraph{Training algorithm}
We train our model as follows:

\begin{algorithm}[htbp]
\caption{Training our model}
\label{alg:train_transformer}
\begin{algorithmic}[1]
\STATE \textbf{Input:} {contexts $D$ and queries $T$ from prior $p(\mathcal{D})$ in dataset}
\STATE \textbf{Output:} {Our model $\tilde{u}_{\theta}$ which can approximate the posterior predictive mean}
\STATE Initialize the model $\Tilde{u}_{\theta}$
\FOR{$i=1$ to $n$}
    \STATE Sample $\boldsymbol{\alpha} \in \Omega$ and $D,T \subseteq \tilde{u}(\boldsymbol{\alpha}) \sim p(\mathcal{D})$ ($D := \{(x_D^{(i)}, t_D^{(i)})\}_{i=1}^{N_D}$, $T := \{(x_T^{(j)}, t_T^{(j)})\}_{j=1}^{N_T}$)
    \IF{dataset $\in$ \{Family of 2D CDR, SWE\}}
        \STATE Regenerate $D^\prime$ and $T^\prime$ by shuffling data points in $D \cup T$  independently within each $(t_n, t_{n+1})$ pair, while preserving $|D^\prime| = N_D$ and $|T^\prime| = N_T$
    \ENDIF
    \STATE Compute MSE loss $\mathcal{L}_{\boldsymbol{\alpha}} = \frac{1}{N_T} \sum_{j=1}^{N_T} \left\{ \tilde{u}_{\theta}(x_T^{(j)}, t_T^{(j)} | D^\prime_n) - \tilde{u}(x_T^{(j)}, t_T^{(j)})\right\}^2$
    \STATE Compute AGF regularization term $\mathcal{L}_{AGF}$ and objective function $\mathcal{L} = \mathcal{L}_{\boldsymbol{\alpha}} + \frac{1}{N_D}\mathcal{L}_{AGF}$
    
    \STATE Update parameters $\theta$ with an AdamW optimizer
\ENDFOR

\end{algorithmic}
\end{algorithm}

\section{Evaluation}
\label{app:Evaluation}
We employ $L_1$ absolute, $L_2$ relative, and $L_\infty$ relative errors between the model's prediction for test queries and the ground truth. The errors are then averaged over the target parameter space or test dataset. Given the true solution $y_{\alpha,k}$ and the corresponding prediction $\hat{y}_{\alpha,k}$ at the $k$-th time point out of a total of $K$ evaluation time points, each metric is computed as follows:
\begin{equation}
\begin{aligned}
L_p \ \text{abs error} = \frac{1}{|\Omega|\cdot K}\sum_{\alpha \in \Omega}\sum_{k=1}^{K} ||y_{\alpha,k}-\hat{y}_{\alpha,k} ||_{p} \ , \ L_p \ \text{rel error} = \frac{1}{|\Omega|\cdot K}\sum_{\alpha \in \Omega}\sum_{k=1}^{K} \frac{||y_{\alpha,k}-\hat{y}_{\alpha,k} ||_{p}}{||y_{\alpha,k}||_{p}}.
\end{aligned}
\end{equation}


\section{Additional Experiments} \label{app:experi}


\subsection{2D CDR Equations with systematically Noisy Prior Data}

In this experiment, we examine whether our model continues to perform well when the prior data is corrupted by systematic synthetic noise instead of being generated by PINNs. Specifically, we construct noisy prior data by injecting either Gaussian noise or uniform noise into the numerical solutions of the family of 2D CDR equations. For the Gaussian-noise setting, we add zero-mean Gaussian noise from $\mathcal{N}(0.0, 0.02)$, and for the uniform-noise setting, we add noise sampled from the uniform distribution $\mathcal{U}([-0.01,0.01])$. We then train our model on these noisy priors and evaluate it using the same procedure as in the main experiments.

\begin{table*}[ht]
\captionsetup{skip=1.5pt}
\caption{The evaluation results of our model trained using systematically noisy data for the spatiotemporal interpolation and temporal extrapolation tasks applied to the family of 2D CDR equations. The text in parentheses following “Ours” indicates the noise type and its standard deviation or range. They are measured at the seen coefficients ($\beta_x, \beta_y, \nu_x, \nu_y, \rho_1, \rho_2, \rho_3 \in \{0.0, 0.5, 1.0\}$). The best performance is marked in \textbf{bold} and the second-best performance is marked with \underline{underline}.}
\label{table:CDR_systemic_noise}
\small
\centering
    \centering
    \captionsetup{labelformat=empty}
    \resizebox{0.9\textwidth}{!}{%
    \begin{tabular}{llccccc}
    \toprule
    \textbf{Task}  & \textbf{Metric} & \textbf{Ours} & \textbf{Ours (PINN)} & \textbf{Ours (Gaussian 0.02)} & \textbf{Ours (Uniform $[-0.01,0.01]$)}  \\ 
    \midrule
    \multirow{3}{*}{\shortstack[l]{Spatiotemporal\\interpolation}} & $L_1$ Abs  & \underline{0.01218}  & 0.01507  & 0.01244  & \textbf{0.01060}  \\
     & $L_2$ Rel & \underline{0.01643}  & {0.02057}  & 0.01670  & \textbf{0.01450} \\
     & $L_\infty$ Rel & \underline{0.04600}  & 0.06584  & {0.04659}   & \textbf{0.04269} \\
    \midrule
    \multirow{3}{*}{\shortstack[l]{Temporal\\extrapolation}} & $L_1$ Abs  & \textbf{0.01474}  & 0.01940  & 0.01562  & \underline{0.01513} \\
    & $L_2$ Rel & \textbf{0.02261}  & 0.02705   & 0.02383   & \underline{0.02320} \\
    & $L_\infty$ Rel & \underline{0.08444}  & 0.08802  & 0.08750  & \textbf{0.08213} \\
    \bottomrule
    \end{tabular}
    }
\end{table*}

Table~\ref{table:CDR_systemic_noise} shows that predictive accuracy is largely maintained under the evaluated zero-mean Gaussian and uniform prior corruptions. In the spatiotemporal interpolation task, the model trained with uniform noise achieves the best performance on all three metrics, even outperforming the version trained on exact numerical priors. In the temporal extrapolation task, the model trained with uniform noise achieves the best performance on $L_\infty$ relative error and the second best performance on two metrics. The model trained with Gaussian noise also remains competitive across all metrics. These results support robustness in the tested noise regimes, without establishing robustness to arbitrary biased, correlated, or localized prior errors.

\subsection{Test Time Evaluation Given PINN Prior in the Family of 2D CDR Equations} \label{app:experi_CDR_given_noisy}

As an additional experiment on the family of 2D CDR equations, we modify the test procedure from Section~\ref{SSS:CDR_seen}. While keeping all other experimental settings the same, we compare the evaluation results when noisy PINN priors are provided as input during testing. This experiment is designed to assess the robustness of the models in producing accurate solutions despite noisy inputs.

\begin{table}[ht]
\captionsetup{skip=1.5pt}
\caption{The evaluation results for the spatiotemporal interpolation and temporal extrapolation task applied to the family of 2D CDR equations given a noisy PINN prior at evaluation. They are measured at the seen coefficients ($\beta_x, \beta_y, \nu_x, \nu_y, \rho_1, \rho_2, \rho_3 \in \{0.0, 0.5, 1.0\}$). The best performance is marked in \textbf{bold} and the second-best performance is marked with \underline{underline}.}
\label{table:CDR_seen_given_noisy}
\small
\centering
    \centering
    \captionsetup{labelformat=empty}
    \resizebox{1.0\columnwidth}{!}{%
    \begin{tabular}{llccccccccc}
    \toprule
    \textbf{Task}  & \textbf{Metric} & \textbf{Ours} & \textbf{Ours (PINN)}  &  \textbf{DeepONet} & \textbf{FNO} & \textbf{F-FNO} & \textbf{A-FNO} & \textbf{LNO}  & \textbf{Poseidon}   & \textbf{DPOT}  \\ 
    \midrule
    \multirow{3}{*}{\shortstack[l]{Spatiotemporal\\interpolation}} & $L_1$ Abs  & \textbf{0.01953}  & \underline{0.02445}  & 0.02540  & - & - & -  & -  & - \\
     & $L_2$ Rel & \textbf{0.02755}  & 0.03467  & \underline{0.03413}  & - & - & - & -  & -  & -  \\
     & $L_\infty$ Rel & \underline{0.10158}  & 0.10716  & \textbf{0.09286}  & - & - & - & -  & -  & -  \\
    \midrule
    \multirow{3}{*}{\shortstack[l]{Temporal\\extrapolation}} & $L_1$ Abs  & \textbf{0.01655}  & \underline{0.02099}  & 0.06888 & 0.03458 & 0.02816 & 0.05794  & 0.03048   & 0.14018 &  0.05051 \\
    & $L_2$ Rel  & \textbf{0.02486}  & \underline{0.02893}   & 0.08591  & 0.04268 & 0.03582 & 0.07903 & 0.03907 & 0.15380  & 0.08624  \\
    & $L_\infty$ Rel  & \underline{0.08830}  &  \textbf{0.09333}  & 0.23386  & 0.10203 & 0.10677 & 0.28805  &  0.11730  & 0.28481  & 0.08624  \\
    \bottomrule
    \end{tabular}%
    }
\end{table}

As shown in the experimental results presented in Table~\ref{table:CDR_seen_given_noisy}, Ours achieves the best performance, demonstrating strong robustness. For Ours (PINN), however, the performance in terms of two relative errors on the spatiotemporal interpolation task is inferior to that of DeepONet. This can be attributed to the fact that Ours (PINN) is pre-trained to predict solutions from PINN priors as inputs, which may cause the model to generate outputs closer to the PINN prior rather than the analytic solution when the prior is provided at test time. Nevertheless, since Ours (PINN) still achieves better performance in terms of $L_1$ absolute error, we can conclude that it also exhibits a certain degree of robustness.


\subsection{Test-Time Prediction with Biased Context Noise in the 2D CDR Family} 
\label{app:experi_biased}

Since Theorem~\ref{thm:mse_ppm} assumes unbiased noise, we further investigate how robust our model remains when the inference stage includes biased Gaussian noise added to the analytic solution. We set the standard deviation of Gaussian noise to $\sigma = 0.01$ and vary only the magnitude of the bias. The evaluation is performed on the temporal extrapolation task of the 2D CDR equations, using the $L_2$ relative error for comparison.

\begin{table}[ht]
\captionsetup{skip=1.5pt}
\caption{The evaluation results for the temporal extrapolation tasks applied to the family of 2D CDR equations. They are measured at the seen coefficients ($\beta_x, \beta_y, \nu_x, \nu_y, \rho_1, \rho_2, \rho_3 \in \{0.0, 0.5, 1.0\}$).}
\label{table:experi_biased}
\small
\centering
    \centering
    \captionsetup{labelformat=empty}
    \resizebox{1.0\columnwidth}{!}{%
    \begin{tabular}{llcccccccccc}
    \toprule
    \textbf{Bias} & \textbf{No noise} & \textbf{0.1\%} & \textbf{0.5\%} & \textbf{1.0\%} & \textbf{1.5\%} & \textbf{2.0\%} & \textbf{2.5\%} & \textbf{3.0\%} & \textbf{3.5\%} & \textbf{4.0\%}   \\ 
    \midrule
    $L_1$ Abs & 0.01474 & 0.01483 & 0.01472 & 0.01481 & 0.01514 & 0.01572 & 0.01664 & 0.01800 & 0.01985 & 0.02212  \\
    $L_2$ Rel & 0.02261  & 0.02271  &  0.02246 & 0.02244  & 0.02271  &  0.02330 & 0.02255  & 0.02573  &  0.02772  &  0.03009  \\
    $L_\infty$ Rel & 0.08444 & 0.08556 & 0.08459 & 0.08426 & 0.08490 & 0.08646 & 0.08901 & 0.09265 & 0.09733 & 0.10276  \\
    \bottomrule
    \end{tabular}%
    }
\end{table}

The experimental results in Table~\ref{table:experi_biased} show that when the bias is below 3\%, the performance degradation remains relatively minor compared to the bias level. However, once the bias reaches approximately 3\% or higher, the error begins to increase more noticeably. This indicates that PDE-PFN maintains robustness against moderate systematic bias, but large bias eventually leads to a degradation in accuracy as expected.


\subsection{Joint Coefficient and Temporal Extrapolation in the 2D CDR Family} 
\label{app:experi_2d_cdr_unseen_domain}

In the experiments on unseen coefficients presented in the main text, the spatiotemporal domain was restricted to the temporal range used during training. This design choice ensured that only the parameter domain included unseen values, preventing both parameter and temporal extrapolation from occurring simultaneously. In this section, we extend the evaluation to a more challenging setting in which both the parameter domain and the temporal domain involve unseen ranges. The unseen coefficient ranges follow those used in Section~\ref{SSS:CDR_unseen}, and for the temporal domain we adopt the evaluation range from the temporal extrapolation task in Section~\ref{SSS:CDR_seen}.

\begin{table}[ht]
\captionsetup{skip=1.5pt}
\caption{The evaluation results for the temporal extrapolation tasks applied to the family of 2D CDR equations. They are measured at the unseen coefficients and unseen temporal domain. The best performance is marked in \textbf{bold} and the second-best performance is marked with \underline{underline}.}
\label{table:CDR_unseen_unseen}
\small
\centering
    \centering
    \captionsetup{labelformat=empty}
    \resizebox{1.0\columnwidth}{!}{%
    \begin{tabular}{lllccccccccc}
    \toprule
    \textbf{Coefficients} & \textbf{Task}  & \textbf{Metric} & \textbf{Ours} & \textbf{Ours (PINN)}  &  \textbf{DeepONet} & \textbf{FNO} & \textbf{F-FNO} & \textbf{A-FNO} & \textbf{LNO} & \textbf{Poseidon}   & \textbf{DPOT}  \\ 
    \midrule
    \multirow{3}{*}{Interpolation} & \multirow{3}{*}{\shortstack[l]{Temporal\\extrapolation}} & $L_1$ Abs  & \textbf{0.00736} & \underline{0.00939} & 0.05139 & 0.10863 & 0.01405 & 0.04695 & 0.01839 & 0.13852 & 0.04610 \\
    & & $L_2$ Rel  & \textbf{0.00954} & \underline{0.01178} & 0.06179 & 0.14380 & 0.01636 & 0.06227 & 0.02049 & 0.14993 & 0.04676  \\
    & & $L_\infty$ Rel  & \textbf{0.03193} & \underline{0.03693} & 0.16008 & 0.34858 & 0.03847 & 0.21750 & 0.04512 & 0.26293 & 0.04676  \\ 
    \midrule
    \multirow{3}{*}{Extrapolation} & \multirow{3}{*}{\shortstack[l]{Temporal\\extrapolation}} & $L_1$ Abs  & \textbf{0.01027} & \underline{0.02063} & 0.05514 & 0.03887 & 0.02312 & 0.05013 & 0.02560 & 0.14661 & 0.05186 \\
    & & $L_2$ Rel  & \textbf{0.01068} & \underline{0.02169} & 0.06450 & 0.05114 & 0.02470 & 0.06351 & 0.02677 & 0.15650 & 0.05189 \\
    & & $L_\infty$ Rel  & \textbf{0.01661} & \underline{0.03474} & 0.16555 & 0.13713 & 0.04591 & 0.20955 & 0.04911 & 0.27502 & 0.05189  \\
    \bottomrule
    \end{tabular}%
    }
\end{table}

The experimental results in Table~\ref{table:CDR_unseen_unseen} show that our model continues to achieve the best performance. In contrast, the baselines experience greater difficulty when both types of extrapolation are required simultaneously, performing noticeably worse than our model in this setting.


\subsection{2D CDR Equations with Different Initial Condition} 
\label{app:experi_2d_cdr_different_initial}

In this experiment, we examine whether our model continues to perform well on the 2D CDR equations from the main paper when the dataset is generated using a different initial condition. Specifically, we replace the original initial condition, $1 + \sin(2x)\sin(2y)$, with $1 + \sin(3x)\sin(3y)$ to construct a new 2D CDR equation dataset. We then conduct training and evaluation following the same procedure used in the main experiments. 

\begin{table}[ht]
\captionsetup{skip=1.5pt}
\caption{The evaluation results for the spatiotemporal interpolation and temporal extrapolation tasks applied to the new family of 2D CDR equations. They are measured in the same way of Section~\ref{SSS:CDR_seen}. The best performance is marked in \textbf{bold} and the second-best performance is marked with \underline{underline}.}
\label{table:CDR_seen2}
\small
\centering
    \centering
    \captionsetup{labelformat=empty}
    \resizebox{1.0\columnwidth}{!}{%
    \begin{tabular}{llccccccccc}
    \toprule
    \textbf{Task}  & \textbf{Metric} & \textbf{Ours} &  \textbf{DeepONet} & \textbf{FNO} & \textbf{F-FNO} & \textbf{A-FNO}  & \textbf{LNO}  & \textbf{Poseidon}   & \textbf{DPOT}  \\ 
    \midrule
    \multirow{3}{*}{\shortstack[l]{Spatiotemporal\\interpolation}} & $L_1$ Abs  & \textbf{0.01339}  & \underline{0.07394}  & \multirow{3}{*}{-} & \multirow{3}{*}{-} & \multirow{3}{*}{-} & \multirow{3}{*}{-} & \multirow{3}{*}{-}  & \multirow{3}{*}{-} \\
     & $L_2$ Rel & \textbf{0.01885}  & \underline{0.09376}  & & & & & & \\
     & $L_\infty$ Rel & \textbf{0.06606}  & \underline{0.26320}  & & & & & &  \\
    \midrule
    \multirow{3}{*}{\shortstack[l]{Temporal\\extrapolation}} & $L_1$ Abs  & \textbf{0.02017}  & 0.03988 & \underline{0.03184} & 0.03371 & 0.08342 & 0.04696 &  0.09397 & 0.12628  \\
    & $L_2$ Rel & \textbf{0.02929}  & 0.05282 & \underline{0.04160} & 0.04192 & 0.09967 & 0.06282 & 0.10352 & 0.13633   \\
    & $L_\infty$ Rel & \textbf{0.10793}  & 0.17487   & 0.12671  & \underline{0.10980} & 0.27716 & 0.18313 & 0.23855 & 0.13633    \\
    \bottomrule
    \end{tabular}%
    }
\end{table}

The experimental results in Table~\ref{table:CDR_seen2} show that our model continues to achieve the best performance even when evaluated on datasets generated with a new initial condition. This confirms that our model does not merely perform well on a single fixed initial condition, but instead generalizes robustly across different initial conditions. Furthermore, the experiments on the SWE and CNSE datasets in the main paper also support this conclusion, as both datasets consist of samples generated from varying initial conditions under the same underlying PDE. Thus, our model is not limited to datasets constructed with a single, unified initial condition.


\subsection{2D CDR Equations with Random Initial Condition Sampled from Uniform Distribution} 
\label{app:experi_2d_cdr_uniform_initial}

In this experiment, we further examine whether our model continues to perform well on the family of 2D CDR equations when the initial condition is no longer fixed and instead varies randomly at the pixel level. Specifically, for each PDE instance, we generate the initial condition by sampling each spatial pixel independently from a uniform distribution. We then conduct training and evaluation following the same procedure used in the main experiments.

\begin{table*}[ht]
\captionsetup{skip=1.5pt}
\caption{The evaluation results for the spatiotemporal interpolation and temporal extrapolation tasks applied to the family of 2D CDR equations with random initial condition sampled from uniform distribution. They are measured at the seen coefficients ($\beta_x, \beta_y, \nu_x, \nu_y, \rho_1, \rho_2, \rho_3 \in \{0.0, 0.5, 1.0\}$). The best performance is marked in \textbf{bold} and the second-best performance is marked with \underline{underline}.}
\label{table:CDR_uniform_initial}
\small
\centering
    \centering
    \captionsetup{labelformat=empty}
    \resizebox{1.0\textwidth}{!}{%
    \begin{tabular}{llccccccccc}
    \toprule
    \textbf{Task}  & \textbf{Metric} & \textbf{Ours} & \textbf{Ours (PINN)} &  \textbf{DeepONet} & \textbf{FNO} & \textbf{F-FNO} & \textbf{A-FNO}  & \textbf{LNO}  & \textbf{Poseidon}   & \textbf{DPOT}  \\ 
    \midrule
    \multirow{3}{*}{\shortstack[l]{Spatiotemporal\\interpolation}} & $L_1$ Abs  & 0.1599  & \underline{0.1513}  & \textbf{0.1401}  & \multirow{3}{*}{-} & \multirow{3}{*}{-} & \multirow{3}{*}{-} & \multirow{3}{*}{-} & \multirow{3}{*}{-}  & \multirow{3}{*}{-} \\
     & $L_2$ Rel & \underline{0.2005}  & \textbf{0.1884}  & 0.2177  & & & & & & \\
     & $L_\infty$ Rel & \underline{0.4253}  & \textbf{0.3874}  & 0.6944  & & & & & &  \\
    \midrule
    \multirow{3}{*}{\shortstack[l]{Temporal\\extrapolation}} & $L_1$ Abs  & \textbf{0.0432}  & {0.0553}  & 0.0788 & 0.0492 & \underline{0.0449} & 0.0710 & 0.0480 & 0.1581 &  0.0720 \\
    & $L_2$ Rel & \textbf{0.0578}  & {0.0719}   & 0.0985  & 0.0641 & \textbf{0.0578} & 0.0940 & \underline{0.0603} & 0.1766  & 0.2117  \\
    & $L_\infty$ Rel & \textbf{0.1588}  & 0.1667   & 0.2706  & 0.1750 & \underline{0.1591} & 0.3008 & 0.1631 & 0.3449  & {0.2117}  \\
    \bottomrule
    \end{tabular}
    }
\end{table*}

The experimental results in Table~\ref{table:CDR_uniform_initial} show that our model continues to achieve strong performance even when the initial condition is randomly sampled from a uniform distribution at the pixel level. In the spatiotemporal interpolation task, Ours (PINN) achieves the best performance on two metrics, while Ours achieves the second-best performance on two of them. In the temporal extrapolation task, Ours achieves the best performance on the three metrics. These results indicate that PDE-PFN remains robust even when the initial condition varies substantially across samples and that its performance is not dependent on a fixed or smooth initial-condition structure.

\subsection{Operator Learning in Airfoil} 
\label{app:experi_Airfoil}

To evaluate input flexibility in operator learning, we conduct experiments on the Airfoil dataset defined on an irregular mesh structure. This dataset consists of 2,000 samples, of which 700 are used for training, 100 for validation, and 200 for testing. The target operator is formulated similarly to that in Section~\ref{SSS:CNSE}: given the initial condition, the model is required to predict the solutions at four time steps, $t = 1.2, 2.4, 3.6,$ and $4.8$ (see Appendix~\ref{app:datasets_airfoil} for a formal description). Most of the baselines used in Section~\ref{S:Experiments} rely on data defined on a regular grid, so that cannot be applied here, since the Airfoil dataset is defined on an irregular mesh. Consequently, among the existing baselines, we only use DeepONet and additionally include Oformer~\citep{li2023transformer} as a new baseline.

\begin{table}[ht]
\captionsetup{skip=1.5pt}
\caption{The evaluation results for the operator learning task applied to the Airfoil dataset. The best performance is marked in \textbf{bold} and the second-best performance is marked with \underline{underline}.}
\label{table:Airfoil}
\small
\centering
    \centering
    \resizebox{0.55\columnwidth}{!}{%
    \captionsetup{labelformat=empty}
    \begin{tabular}{llcccc}
    \toprule
    \textbf{Metric} & \textbf{Ours} & \textbf{Ours (PINN)}  &  \textbf{DeepONet} & \textbf{Oformer}  \\ 
    \midrule
    $L_1$ Abs  & \underline{0.11047}  & \textbf{0.10755} & 0.23160  & 0.12039	 \\
    $L_2$ Rel & \underline{0.18831}  & \textbf{0.18357}  & 0.44101  & 0.19744   \\
    $L_\infty$ Rel &   0.33584 & \underline{0.32887}  & 0.73076  & \textbf{0.32318}   \\
    \bottomrule
    \end{tabular}%
    }
\end{table}
The difference from the experimental results reported in the original Oformer paper arises because their setting involved using the first four time steps as input to predict subsequent time steps, whereas our experiment is formulated based on the initial condition only. The experimental results in Table~\ref{table:Airfoil} show that both versions of our model outperform the baselines on two metrics. These results demonstrate that our model's ICL capability effectively extends to irregular mesh data and can generalize well to new PDE problems. Detailed information on the best hyperparameter settings can be found in Appendix~\ref{app:hyperparameter}.

\subsection{Operator Learning in Darcy Flow} \label{app:experi_Darcy}
To further evaluate the PDE generalization ability of our model in operator learning, we conducted additional experiments on the Darcy Flow dataset. This dataset consists of 2,000 samples, with 1,400 used for training, 200 for validation, and 400 for testing. The operator learning task is defined as predicting the steady-state solution $u(x)$ from a given viscosity field $a(x)$ (see Appendix~\ref{app:datasets_Darcy} for a formal description). Unlike the CNSE dataset, where both context and query represent the same feature, the Darcy Flow dataset involves mapping between different quantities. This makes it a distinct form of operator learning task.

\begin{table}[ht]
\captionsetup{skip=1.5pt}
\caption{The evaluation results for the operator learning task applied to Darcy Flow. The best performance is marked in \textbf{bold} and the second-best performance is marked with \underline{underline}.}
\label{table:Darcy}
\small
\centering
    \centering
    \captionsetup{labelformat=empty}
    \resizebox{1.0\columnwidth}{!}{%
    \begin{tabular}{lccccccccc}
    \toprule
    \textbf{Metric} & \textbf{Ours} & \textbf{Ours (PINN)}  &  \textbf{DeepONet} & \textbf{FNO} & \textbf{F-FNO} & \textbf{A-FNO} & \textbf{LNO}  & \textbf{Poseidon}  & \textbf{DPOT}  \\ 
    \midrule
    $L_1$ Abs  & 0.00796  & \underline{0.00742} & 0.00996  & 0.00792	 & 0.00806	 & 0.00906	& 0.01413 & 0.00998  & \textbf{0.00726}  \\
    $L_2$ Rel & \underline{0.28771}  & \textbf{0.28091}  & 0.40759  & 0.29292 & 0.34149 & 0.40357 & 0.51778 & 0.41001  & 0.51505  \\
    $L_\infty$ Rel &   \textbf{0.39694} & \underline{0.40124}  & 0.59057  & 0.46398 & 0.59381 & 0.80674 & 0.61245 & 0.52933  & 0.51505   \\
    \bottomrule
    \end{tabular}%
    }
\end{table}

The experimental results in Table~\ref{table:Darcy} show that both versions of our model outperform the baselines on two metrics. These findings further confirm that our model maintains strong PDE generalization ability even when the operator learning task requires transferring between different physical quantities. Detailed information on the best hyperparameter settings can be found in Appendix~\ref{app:hyperparameter}.


\subsection{Spatiotemporal Interpolation and Temporal extrapolation in 3D Diffusion-Reaction equations} 
\label{app:experi_3d_dr}

To further explore the potential of our model for 3D PDEs, we conduct additional experiments on 3D diffusion–reaction equations. We construct a 3D diffusion–reaction equation dataset by extending the same procedure used for generating the 2D CDR equations. This formulation includes three diffusion coefficients $(\nu_x, \nu_y, \nu_z)$ and one reaction coefficient $\rho$, each sampled from ${0.0, 0.5, 1.0, 1.5}$, resulting in 256 distinct equations. We perform experiments on the same two tasks as in the 2D CDR equations—spatiotemporal interpolation and temporal extrapolation—using DeepONet, FNO, and DPOT as baselines, since only they can handle 3D datasets among baselines.

\begin{table}[ht]
\captionsetup{skip=1.5pt}
\caption{The evaluation results for the spatiotemporal interpolation and temporal extrapolation tasks applied to the 3D diffusion reaction equations. They are measured at the seen coefficients ($\nu_x, \nu_y, \nu_z, \rho_1 \in \{0.0, 0.5, 1.0, 1.5\}$). The best performance is marked in \textbf{bold} and the second-best performance is marked with \underline{underline}.}
\label{table:3D_dr}
\small
\centering
    \centering
    \captionsetup{labelformat=empty}
    \resizebox{0.65\columnwidth}{!}{%
    \begin{tabular}{llccccc}
    \toprule
    \textbf{Task}  & \textbf{Metric} & \textbf{Ours}  & \textbf{Ours (PINN)} &  \textbf{DeepONet} & \textbf{FNO}  & \textbf{DPOT}  \\ 
    \midrule
    \multirow{3}{*}{\shortstack[l]{Spatiotemporal\\interpolation}} & $L_1$ Abs  & \textbf{0.02685}  & 0.02843 & \underline{0.02816} & \multirow{3}{*}{-} & \multirow{3}{*}{-}  \\
     & $L_2$ Rel & \textbf{0.03711} & 0.04144 & \underline{0.03981}  &   &  \\
     & $L_\infty$ Rel & \textbf{0.10938} & 0.12823 & \underline{0.12724} &  &  \\
    \midrule
    \multirow{3}{*}{\shortstack[l]{Temporal\\extrapolation}} & $L_1$ Abs  & \textbf{0.01031} & 0.01039  & 0.02421  & 0.01521 & \underline{0.01485}  \\
    & $L_2$ Rel & \textbf{0.01376} & \underline{0.01391} & 0.03033   & 0.01824  &   0.01638 \\
    & $L_\infty$ Rel & \underline{0.04180} & 0.04545 & 0.11530   & 0.05545  & \textbf{0.01638} \\
    \bottomrule
    \end{tabular}%
    }
\end{table}

The experimental results in Table~\ref{table:3D_dr} show that our model outperforms the baselines in spatiotemporal interpolation in all metrics and in temporal extrapolation on two metrics, demonstrating its potential to generalize effectively to 3D datasets.

\subsection{Operator Learning in 3D Gray-Scott equation} 
\label{app:experi_3d_gray_scott}
To further evaluate the PDE generalization ability of our model in a more challenging 3D setting, we conduct additional experiments on the 3D Gray-Scott equations from APEBench~\cite{koehler2024apebench}. In this experiment, we focus on the operator learning task, where the models are required to learn the one-step time-advancing operator of the system. Specifically, let $u(x,t)=[u_0(x,t),u_1(x,t)] \in \mathbb{R}^2$ denote the two-channel state of the Gray-Scott system at spatial coordinate $x$. Given the state $u(x,t)$ at a current time step $t$, the model is asked to predict the state at the next time step $u(x,t+\Delta t)$. We compare our method with DeepONet, FNO, and DPOT, which are widely used baselines for operator learning, and additionally report the performance of our model trained with PINN priors.

\begin{table}[ht]
\captionsetup{skip=1.5pt}
\caption{The evaluation results for the operator learning task in 3D Gray-Scott equation. The best performance is marked in \textbf{bold} and the second-best performance is marked with \underline{underline}.}
\label{table:3D_gray_scott}
\small
\centering
    \centering
    \captionsetup{labelformat=empty}
    \resizebox{0.65\columnwidth}{!}{%
    \begin{tabular}{lccccc}
    \toprule
    \textbf{Metric} & \textbf{Ours}  & \textbf{Ours (PINN)} &  \textbf{DeepONet} & \textbf{FNO}  & \textbf{DPOT}  \\ 
    \midrule
    $L_1$ Abs      & \textbf{0.00152} & \underline{0.00177} & 0.01622   & 0.00558  & 0.01458  \\
    $L_2$ Rel      & \textbf{0.02626} & \underline{0.03004} & 0.24546   & 0.08019  & 0.22860 \\
    $L_\infty$ Rel & \textbf{0.07482} & \underline{0.10394} & 0.57927   & 0.19076  & 0.26938 \\
    \bottomrule
    \end{tabular}
    }
\end{table}

The experimental results in Table~\ref{table:3D_gray_scott} show that both versions of our model substantially outperform all baselines across all evaluation metrics. In particular, Ours achieves the best performance on all three metrics, while Ours (PINN) consistently achieves the second-best performance. These results demonstrate that PDE-PFN can generalize effectively to a complex 3D system and can learn the corresponding operator with high accuracy even when trained with noisy PINN priors. The strong performance gap over the baselines further supports the effectiveness of our in-context learning framework for higher-dimensional PDE problems.


\subsection{Sensitivity Analysis to the Number of Test Context Points in the Family of 2D CDR Equation} 
\label{app:experi_Sensitivity}

In this section, we empirically examine how PDE-PFN behaves as the number of test context $\overline{D}$ observations increases. Specifically, we evaluate the sensitivity of the neural network's posterior predictive mean to the number of the test context set $\overline{D}$ provided during the test process.

To this end, we conduct an evaluation of spatiotemporal interpolation on the family of 2D CDR equations experiment by varying the number of $\overline{D}$ while keeping other factors unchanged, including the best hyperparameter settings. The results, depicted in Figure~\ref{fig:Sensitivity}, clearly demonstrate that as the size of $\overline{D}$ increases, the error consistently decreases across all three evaluation metrics. This trend is qualitatively consistent with Corollary~\ref{cor:context_size_bayes_risk}, which concerns non-increasing population-optimal MSE for nested context sets rather than guaranteed monotonic improvement of a finite trained model.

\begin{figure}[htbp]
    \centering
    \includegraphics[width=1.0\linewidth]{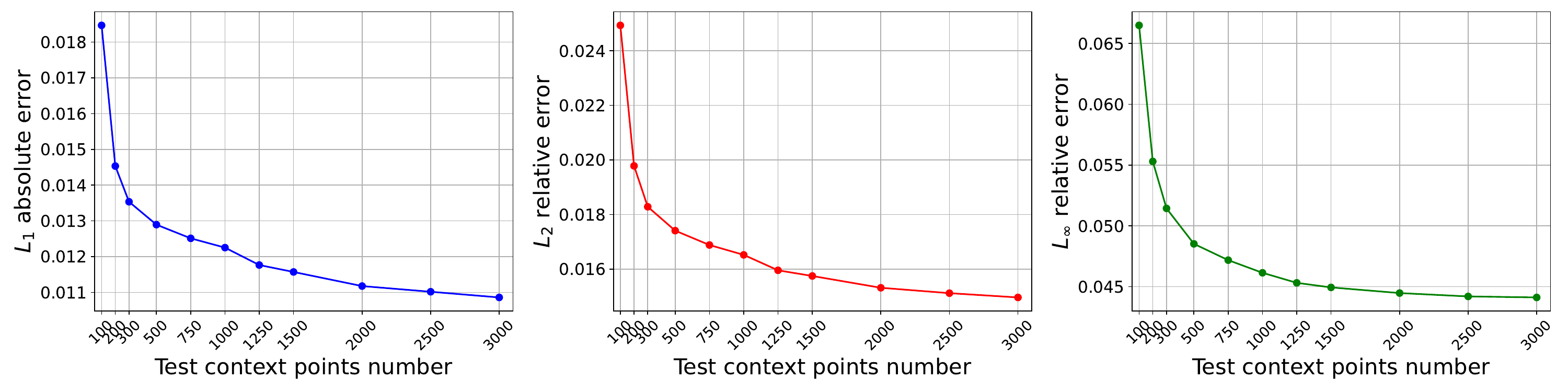}
    \caption{Sensitivity to the number of test context points $\overline{D}$ in the family of 2D CDR equations. Prediction errors decrease as more context observations are provided, which is consistent with the posterior-predictive-mean interpretation under squared loss.}
    \label{fig:Sensitivity}
\end{figure}

These empirical findings are consistent with Corollary~\ref{cor:context_size_bayes_risk}, which states that the population-optimal MSE is non-increasing as more context observations are provided. We emphasize that this experiment does not verify convergence of the full posterior distribution. Rather, it shows that PDE-PFN benefits from more informative context sets, as expected from the PPM interpretation under squared loss. Overall, the decreasing error trend suggests that PDE-PFN effectively benefits from more informative context sets, consistent with the PPM interpretation under squared loss.


\subsection{Ablation Study} 
\label{app:experi_ablation}

In this section, we conduct an ablation study to analyze the necessity of the key components in our model. Using the 2D CDR equation experiments, we replace AGF with a \textit{vanilla} Transformer layer and substitute the rational activation with GELU to evaluate the impact of each component on performance.

\begin{table}[ht]
\captionsetup{skip=1.5pt}
\caption{The evaluation results for the spatiotemporal interpolation and temporal extrapolation tasks applied to the family of 2D CDR equations. “Ours w.o. agf,” “Ours w.o. act,” and “Ours w.o. all” denote replacing AGF with a vanilla Transformer block, replacing the rational activation with GELU, and replacing both, respectively. They are measured at the seen coefficients ($\beta_x, \beta_y, \nu_x, \nu_y, \rho_1, \rho_2, \rho_3 \in \{0.0, 0.5, 1.0\}$). The best performance is marked in \textbf{bold} and the second-best performance is marked with \underline{underline}.}
\label{table:experi_ablation1}
\small
\centering
    \centering
    \captionsetup{labelformat=empty}
    \resizebox{0.65\columnwidth}{!}{%
    \begin{tabular}{llcccc}
    \toprule
    \textbf{Task}  & \textbf{Metric} & \textbf{Ours} &  \textbf{Ours w.o. agf} & \textbf{Ours w.o. act}  & \textbf{Ours w.o. all}  \\ 
    \midrule
    \multirow{3}{*}{\shortstack[l]{Spatiotemporal\\interpolation}} & $L_1$ Abs  & \underline{0.01218}  & 0.01242 & \textbf{0.01040} & 0.01780 \\
     & $L_2$ Rel & 0.01643  & \underline{0.01625}  &  \textbf{0.01495} & 0.02281 \\
     & $L_\infty$ Rel & 0.04600  & \textbf{0.04297} & \underline{0.04399} & 0.05736 \\
    \midrule
    \multirow{3}{*}{\shortstack[l]{Temporal\\extrapolation}} & $L_1$ Abs  & \textbf{0.01474}  & 0.01776  & \underline{0.01725} & 0.02454  \\
    & $L_2$ Rel & \textbf{0.02261}  & 0.02779   & \underline{0.02665}  & 0.03437  \\
    & $L_\infty$ Rel & \textbf{0.08444} & 0.09554   & \underline{0.09376} & 0.10625 \\
    \bottomrule
    \end{tabular}%
    }
\end{table}

The experimental results in Table~\ref{table:experi_ablation1} show that although the ablated variants may occasionally perform slightly better on the interpolation task, our full model consistently achieves superior performance on the more challenging temporal extrapolation task. This supports our claim that the proposed architectural components enhance the model’s ability to generalize across diverse tasks in a stable and efficient manner.

We additionally examine whether the robustness observed in Section~\ref{SS:CDR} is preserved by conducting the same PINN-prior evaluation on the “Ours with sinusoidal activation’’ and “Ours w.o. AGF’’ variants.

\begin{table}[ht]
\captionsetup{skip=1.5pt}
\caption{The evaluation results for the spatiotemporal interpolation and temporal extrapolation tasks applied to the family of 2D CDR equations. “Ours w.o. agf,” “Ours w.o. act,” and “Ours w.o. all” denote replacing AGF with a vanilla Transformer block, replacing the rational activation with GELU, and replacing both, respectively. They are measured at the seen coefficients ($\beta_x, \beta_y, \nu_x, \nu_y, \rho_1, \rho_2, \rho_3 \in \{0.0, 0.5, 1.0\}$). The best performance is marked in \textbf{bold} and the second-best performance is marked with \underline{underline}.}
\label{table:experi_ablation2}
\small
\centering
    \centering
    \captionsetup{labelformat=empty}
    \resizebox{0.65\columnwidth}{!}{%
    \begin{tabular}{llccc}
    \toprule
    \textbf{Task}  & \textbf{Metric} & \textbf{Ours} &  \textbf{Ours w.o. agf} & \textbf{Ours with sinusoidal}   \\ 
    \midrule
    \multirow{3}{*}{\shortstack[l]{Spatiotemporal\\interpolation}} & $L_1$ Abs  & \textbf{0.01507}  & 0.01667 & \underline{0.01482}  \\
     & $L_2$ Rel & \textbf{0.02057}  & 0.02171  &  \underline{0.02096}  \\
     & $L_\infty$ Rel & \underline{0.06584}  & \textbf{0.05536} & {0.06946}  \\
    \midrule
    \multirow{3}{*}{\shortstack[l]{Temporal\\extrapolation}} & $L_1$ Abs  & \textbf{0.01940}  & \underline{0.02105}  & 0.02671   \\
    & $L_2$ Rel & \textbf{0.02705}  & \underline{0.03030}   & 0.04159    \\
    & $L_\infty$ Rel & \textbf{0.08802} & \underline{0.09518}   & 0.14922  \\
    \bottomrule
    \end{tabular}%
    }
\end{table}

Table~\ref{table:experi_ablation2} shows that our model achieves the best results on both tasks, confirming strong robustness, whereas both ablations exhibit degraded performance. AGF and rational activation are most beneficial for temporal extrapolation and robustness under PINN-generated prior data, while interpolation performance can be achieved by simpler variants. We therefore view these components as stability-oriented rather than universally performance-dominating.


\subsection{Computational Cost Analysis related to the Number of Context and Query Points} 
\label{app:experi_cost_analysis}

In this section, we analyze how the number of context and query points affects the training and inference time as well as memory cost. For both training and inference, we report the average per-iteration runtime measured over 10 iterations and record the peak GPU memory usage. When evaluating the effect of context size, we fix the number of query points to 4,096. Likewise, when evaluating the effect of query size, we fix the number of context points to 4,096. Training and inference are profiled under BF16 mixed precision in this section.

\begin{table}[ht]
\captionsetup{skip=1.5pt}
\caption{Training and inference performance with different \textbf{context lengths}.}
\label{table:experi_context_length}
\small
\centering
\captionsetup{labelformat=empty}
\resizebox{0.7\columnwidth}{!}{%
\begin{tabular}{lcccccc}
\toprule
\textbf{Context length} & \textbf{1,024} & \textbf{2,048} & \textbf{4,096} & \textbf{8,192} & \textbf{12,288} & \textbf{16,384} \\
\midrule
Training time (sec)   & 0.104 & 0.102 & 0.101 & 0.107 & 0.136 & 0.165 \\
Training memory (MB)  & 280.0 & 329.8 & 420.7 & 617.2 & 800.4 & 987.4 \\
Inference time (sec)  & 0.015 & 0.015 & 0.022 & 0.037 & 0.052 & 0.069 \\ 
Inference memory (MB) & 182.1 & 228.7 & 323.6 & 510.7 & 702.1 & 889.7 \\
\bottomrule
\end{tabular}
}
\end{table}

\begin{table}[ht]
\captionsetup{skip=1.5pt}
\caption{Training and inference performance with different \textbf{query lengths}.}
\label{table:experi_query_length}
\small
\centering
\captionsetup{labelformat=empty}
\resizebox{0.7\columnwidth}{!}{%
\begin{tabular}{lcccccc}
\toprule
\textbf{Query length} & \textbf{1,024} & \textbf{2,048} & \textbf{4,096} & \textbf{8,192} & \textbf{12,288} & \textbf{16,384} \\
\midrule
Training time (sec)   & 0.101 & 0.100 & 0.101 & 0.123 & 0.148 & 0.176 \\
Training memory (MB)  & 872.7 & 1010.1 & 1265.4 & 1788.6 & 2303.0 & 2831.2 \\
Inference time (sec)  & 0.02 & 0.02 & 0.021 & 0.024 & 0.027 & 0.03 \\
Inference memory (MB) & 242.3 & 270.1 & 323.1 & 433.6 & 540.7 & 648.1 \\
\bottomrule
\end{tabular}
}
\end{table}

As shown in Tables~\ref{table:experi_context_length} and~\ref{table:experi_query_length}, runtime and memory show approximately linear scaling with context or query length. These empirical measurements characterize the evaluated configuration and do not by themselves establish asymptotic computational complexity.

\subsection{Context Identifiability in the CDR Family}

We examine whether the model can distinguish dynamics when the initial context is identical. We select two seen CDR instances,
$$ q_1=(0,0,0.5,0.5,0.5,0.5,0.5),\qquad q_2=(1,1,0.5,0.5,0.5,0.5,0.5), $$
with coefficient order $(\beta_x,\beta_y,\nu_x,\nu_y,\rho_1,\rho_2,\rho_3)$. These instances have the same field at $t=0$ but evolve differently. We fix $N_D=1,024$ and start from the complete initial field. For each $m\in\{0,4,16,32,64\}$, we replace $m$ randomly selected initial observations with observations at $t=0.1$ at the same spatial locations. Both the time coordinate and the observed value are updated. We then predict the complete $32\times32$ field at $t=0.2$ without gradient updates.

To quantify how strongly the predictions distinguish the instances, we define separation recovery as
$$ R_{\mathrm{sep}}= \frac{\|\hat{u}_{q_1}-\hat{u}_{q_2}\|_2} {\|u_{q_1}-u_{q_2}\|_2}. $$
Here, $R_{\mathrm{sep}}=0$ indicates identical predictions, whereas $R_{\mathrm{sep}}=1$ indicates recovery of the reference separation magnitude, not necessarily correct prediction fields. Results for $m>0$ are reported as mean $\pm$ standard deviation over ten random observation selections; $m=0$ has no selection variability.

\begin{table}[ht]
\captionsetup{skip=1.5pt}
\caption{Context identifiability with fixed context cardinality. Replacing initial observations with later observations introduces information about the differing dynamics. Standard deviations reflect observation-selection variability, not training-seed variability.}
\label{table:context_identifiability}
\small
\centering
\captionsetup{labelformat=empty}
\resizebox{0.6\columnwidth}{!}{%
\begin{tabular}{lccc}
\toprule
\textbf{$m$} & \textbf{Separation recovery} & \textbf{$L_2$ Rel: $q_1$} & \textbf{$L_2$ Rel: $q_2$}  \\
\midrule
0 & 0.000 & 0.3059 & 0.2649 \\
4 &  0.073 $\pm$ 0.195 & 0.3041 $\pm$ 0.0021 & 0.2638 $\pm$ 0.0012 \\
16 & 0.270 $\pm$ 0.310 & 0.2970 $\pm$ 0.0057 & 0.2600 $\pm$ 0.0017 \\
32 & 0.570 $\pm$ 0.337 & 0.2874 $\pm$ 0.0060 & 0.2550 $\pm$ 0.0013 \\
64 & 0.887 $\pm$ 0.342 & 0.2692 $\pm$ 0.0109 & 0.2465 $\pm$ 0.0025 \\
\bottomrule
\end{tabular}
}
\end{table}

With identical initial contexts, a deterministic predictor necessarily produces identical predictions. Introducing informative observations increases separation recovery and reduces both prediction errors, although substantial error remains. The variability across selections indicates that observation location matters. This experiment illustrates that prediction depends on distinctions present in the observed context; flexible context input does not resolve unobserved ambiguity.

\subsection{Sensitivity to Context Time}

We evaluate context-time sensitivity for seen-coefficient CDR temporal prediction without retraining. The context contains the complete $32\times32$ spatial state, so $N_D=1,024$ remains fixed. We vary the context time over $t_c\in\{0.45,0.50,0.60,0.65\}$ while keeping the query fields fixed at $t\in\{0.7,0.8,0.9,1.0\}$. All queried values are later than the context and remain masked in the model input.

\begin{table}[ht]
\captionsetup{skip=1.5pt}
\caption{Sensitivity to the time of the observed context snapshot, with fixed context cardinality and future query times. The same trained checkpoint is used throughout.}
\label{table:sesitivity_Context_Time}
\small
\centering
\captionsetup{labelformat=empty}
\resizebox{0.5\columnwidth}{!}{%
\begin{tabular}{lccc}
\toprule
\textbf{Context time $t_c$} & \textbf{$L_1$ Abs} & \textbf{$L_2$ Rel} & \textbf{$L_\infty$ Rel}  \\
\midrule
0.45 & 0.01458 & 0.02199 & 0.08245 \\
0.50 & 0.01437 & 0.02165 & 0.08024 \\
0.60 & 0.01474 & 0.02261 & 0.08444 \\
0.65 & 0.01604 & 0.02472 & 0.09530 \\
\bottomrule
\end{tabular}
}
\end{table}

Errors remain similar across the evaluated snapshots but do not decrease monotonically as the context approaches the query times. Contexts at $t_c=0.45$ and $0.50$ lie within the training-time domain, whereas those at $0.60$ and $0.65$ lie outside it; this distributional difference may contribute to the observed pattern. The results demonstrate flexibility across these complete snapshots, rather than invariance to context selection or robustness to all observation patterns.

\subsection{Periodic Boundary Consistency}

We evaluate periodic boundary consistency in the seen-coefficient CDR temporal extrapolation setting. Given the observed state $u(\cdot,0.6)$, we independently query exact endpoint pairs at $t\in\{0.7,0.8,0.9,1.0\}$ and examine the differences
$$ \delta_x(y,t)=\hat{u}(0,y,t)-\hat{u}(\pi,y,t),\qquad \delta_y(x,t)=\hat{u}(x,0,t)-\hat{u}(x,\pi,t). $$
The endpoint $\pi$ is absent from the stored grid, making this an off-grid query diagnostic. DeepONet is evaluated at the same endpoint pairs. Neither model is retrained or supplied with explicit boundary-condition constraints for this evaluation.

\begin{table}[ht]
\captionsetup{skip=1.5pt}
\caption{Periodic endpoint mismatch for seen-coefficient CDR temporal extrapolation. Lower values indicate greater consistency of independently predicted periodic endpoint pairs.}
\label{table:periodic_boundary_consistency}
\small
\centering
\captionsetup{labelformat=empty}
\resizebox{0.6\columnwidth}{!}{%
\begin{tabular}{lccc}
\toprule
\textbf{Model} & \textbf{Periodic BC RMSE} & \textbf{Relative periodic BC error}  \\
\midrule
DeepONet & 0.07097 & 0.06898 \\
PDE-PFN  & \textbf{0.02081} & \textbf{0.02119} \\
\bottomrule
\end{tabular}
}
\end{table}

PDE-PFN exhibits lower periodic mismatch in this evaluated setting. This provides evidence of approximate periodic consistency of its predictions, but not exact boundary-condition enforcement, PDE residual satisfaction, or conservation-law guarantees. The observed state is supplied as context rather than reconstructed, so this experiment does not evaluate initial-condition prediction error.

\subsection{Comparison with GNOT} \label{app:gnot_comparison}

\paragraph{Experimental protocol.}
We compare PDE-PFN with GNOT~\cite{hao2023gnot}, a Transformer-based neural operator that supports irregular meshes and pointwise prediction. The comparison matches datasets, splits, observed solution information, target coordinates, training-example counts, optimization budget, and evaluation metrics. We report the PDE-PFN configuration trained on numerical priors. The CDR experiments use the same pre-trained PDE-PFN checkpoint without task-specific gradient updates, whereas the SWE and CNSE experiments use downstream fine-tuning as in the main experiments. For unseen-coefficient CDR temporal extrapolation, the context is at $t=0.1$ and queries are at $t\in\{0.2,0.3,0.4,0.5\}$, within the training-time domain.

\begin{table}[ht]
\captionsetup{skip=1.5pt}
\caption{Comparison with GNOT on the 2D CDR family. PDE-PFN uses numerical priors. temporal extrapolation with unseen coefficients remains within the training-time domain. Lower errors are better; the lower value for each metric in each setting is bold.}
\label{tab:gnot_cdr}
\small
\centering
\captionsetup{labelformat=empty}
\resizebox{0.8\columnwidth}{!}{%
\begin{tabular}{lcccc}
\toprule
Task & Model & $L_1$ Abs & $L_2$ Rel & $L_\infty$ Rel \\
\midrule
\multirow{2}{*}{\textit{Seen coefficients: spatiotemporal interpolation}} &
PDE-PFN & \textbf{0.01218} & \textbf{0.01643} & \textbf{0.04600} \\
& GNOT    & 0.01477 & 0.01928 & 0.04619 \\
\midrule
\multirow{2}{*}{\textit{Seen coefficients: temporal extrapolation}} &
PDE-PFN & \textbf{0.01474} & \textbf{0.02261} & \textbf{0.08444} \\
& GNOT    & 0.02377 & 0.03126 & 0.08807 \\
\midrule
\multirow{2}{*}{\textit{Coefficient interpolation: spatiotemporal interpolation}} &
PDE-PFN & \textbf{0.01241} & 0.01658 & 0.04274 \\
& GNOT    & 0.01269 & \textbf{0.01657} & \textbf{0.03930} \\
\midrule
\multirow{2}{*}{\textit{Coefficient interpolation: temporal extrapolation}} &
PDE-PFN & \textbf{0.08957} & \textbf{0.11990} & \textbf{0.29526} \\
& GNOT    & 0.15797 & 0.22053 & 0.57981 \\
\midrule
\multirow{2}{*}{\textit{Coefficient extrapolation: spatiotemporal interpolation}} &
PDE-PFN & \textbf{0.01661} & \textbf{0.02398} & \textbf{0.08324} \\
& GNOT    & 0.04139 & 0.05339 & 0.12084 \\
\midrule
\multirow{2}{*}{\textit{Coefficient extrapolation: temporal extrapolation}} &
PDE-PFN & \textbf{0.01027} & \textbf{0.01068} & \textbf{0.01661} \\
& GNOT    & 0.06625 & 0.08667 & 0.22352 \\
\bottomrule
\end{tabular}
}
\end{table}

\begin{table}[ht]
\captionsetup{skip=1.5pt}
\caption{Comparison with GNOT on SWE after downstream fine-tuning. PDE-PFN is initialized from pre-training on numerical CDR priors. Seen/unseen refers to initial-condition samples used/not used in downstream training. Lower errors are better; the lower value for each metric in each setting is bold.}
\label{tab:gnot_swe}
\small
\centering
\captionsetup{labelformat=empty}
\resizebox{0.8\columnwidth}{!}{%
\begin{tabular}{lcccc}
\toprule
Task & Model & $L_1$ Abs & $L_2$ Rel & $L_\infty$ Rel \\
\midrule
\multirow{2}{*}{\textit{Seen initial conditions: spatiotemporal interpolation}} &
PDE-PFN & \textbf{0.00208} & \textbf{0.00739} & \textbf{0.05897} \\
& GNOT    & 0.00317 & 0.00816 & 0.06265 \\
\midrule
\multirow{2}{*}{\textit{Seen initial conditions: temporal extrapolation}} &
PDE-PFN & \textbf{0.01548} & \textbf{0.03273} & \textbf{0.20590} \\
& GNOT    & 0.03076 & 0.07564 & 0.42872 \\
\midrule
\multirow{2}{*}{\textit{Unseen initial conditions: spatiotemporal interpolation}} &
PDE-PFN & \textbf{0.00213} & \textbf{0.00754} & \textbf{0.05986} \\
& GNOT    & 0.00317 & 0.00818 & 0.06233 \\
\midrule
\multirow{2}{*}{\textit{Unseen initial conditions: temporal extrapolation}} &
PDE-PFN & \textbf{0.01571} & \textbf{0.03323} & \textbf{0.20686} \\
& GNOT    & 0.03106 & 0.07636 & 0.43261 \\
\bottomrule
\end{tabular}
}
\end{table}

\begin{table}[ht]
\captionsetup{skip=1.5pt}
\caption{Comparison with GNOT on CNSE trajectory prediction. PDE-PFN is fine-tuned from pre-training on numerical CDR priors. Lower errors are better; the lower value for each metric is bold.}
\label{tab:gnot_cnse}
\small
\centering
\captionsetup{labelformat=empty}
\resizebox{0.45\columnwidth}{!}{%
\begin{tabular}{lcccc}
\toprule
Model & $L_1$ Abs & $L_2$ Rel & $L_\infty$ Rel \\
\midrule
PDE-PFN & \textbf{0.09497} & \textbf{0.54833} & \textbf{0.63533} \\
GNOT    & 0.09601 & 0.57140 & 0.65714 \\
\bottomrule
\end{tabular}
}
\end{table}

\paragraph{Results.}
Tables~\ref{tab:gnot_cdr}--\ref{tab:gnot_cnse} report the CDR, SWE, and CNSE comparisons. PDE-PFN achieves lower errors on all three metrics in 10 of the 11 settings. The exception is CDR spatiotemporal interpolation at interpolated coefficients: PDE-PFN has lower $L_1$ error, whereas GNOT has slightly lower $L_2$ error
($0.01657$ versus $0.01658$) and lower $L_\infty$ error. PDE-PFN improves all three metrics in the other five CDR
settings, all four SWE settings, and CNSE. These results extend the comparison to a related context-query neural operator. However, the models differ in both training formulation and architecture, so this comparison does not isolate the contribution of PFN-style prior fitting from architectural design.


\section{Detailed Results}
\subsection{Detailed Experimental Results of PDE-PFN with Standard Deviation} 
\label{app:experi_detail_results}

In this section, we provide more detailed results for the family of 2D CDR equation experiments presented in Section~\ref{SS:CDR}. To further assess the stability of our model, we report the results over 5 random seeds in the form of mean $\pm$ standard deviation. As shown in Table~\ref{table:detailed_results}, our model demonstrates consistently stable performance across all evaluation settings, including seen coefficients, coefficient interpolation, and coefficient extrapolation, for both spatiotemporal interpolation and temporal extrapolation tasks. These results further support that the strong performance of PDE-PFN is not due to a particular random seed, but is consistently observed across repeated runs.

\begin{table}[ht]
\captionsetup{skip=1.5pt}
\caption{Detailed results of our model on the family of 2D CDR equations. All results are reported as mean $\pm$ standard deviation over 5 random seeds.}
\label{table:detailed_results}
\small
\centering
    \centering
    \captionsetup{labelformat=empty}
    \resizebox{\columnwidth}{!}{%
    \begin{tabular}{lcccccc}
    \toprule
    \multirow{2}{*}{\textbf{Metric}} & \multicolumn{2}{c}{\textbf{Seen Coeffcient}} & \multicolumn{2}{c}{\textbf{Coeff Interpolation}}  & \multicolumn{2}{c}{\textbf{Coeff Extrapolation}}   \\ 
    \cmidrule(lr){2-3} \cmidrule(lr){4-5} \cmidrule(lr){6-7}
    & \shortstack[c]{\textbf{Spatiotemporal}\\\textbf{Interpolation}} & \shortstack[c]{\textbf{Temporal}\\\textbf{Extrapolation}} & \shortstack[c]{\textbf{Spatiotemporal}\\\textbf{Interpolation}} & \shortstack[c]{\textbf{Temporal}\\\textbf{Extrapolation}} & \shortstack[c]{\textbf{Spatiotemporal}\\\textbf{Interpolation}} & \shortstack[c]{\textbf{Temporal}\\\textbf{Extrapolation}}  \\ 
    \midrule
    $L_1$ Abs      & 0.01218 $\pm$ 0.00218 & 0.01474 $\pm$ 0.00346 & 0.01241 $\pm$ 0.00170 & 0.08957 $\pm$ 0.01627 & 0.01661 $\pm$ 0.00132 & 0.01027 $\pm$ 0.00348  \\
    $L_2$ Rel      & 0.01643 $\pm$ 0.00250 & 0.02261 $\pm$ 0.00449 & 0.01658 $\pm$ 0.00176 & 0.11990 $\pm$ 0.02624 & 0.02398 $\pm$ 0.00202 & 0.01068 $\pm$ 0.00443 \\
    $L_\infty$ Rel & 0.04600 $\pm$ 0.00383 & 0.08444 $\pm$ 0.01088 & 0.04274 $\pm$ 0.00036 & 0.29526 $\pm$ 0.08702 & 0.08324 $\pm$ 0.00441 & 0.01661 $\pm$ 0.01137 \\
    \bottomrule
    \end{tabular} 
    }
\end{table}

\subsection{Qualitative results}
Figures~\ref{visualization_2D_CDR} and~\ref{visualization_2D_SWE} present qualitative comparisons for the 2D CDR and SWE experiments, respectively. For the CDR equation, PDE-PFN closely reconstructs the spatial patterns of the ground-truth solutions in both spatiotemporal interpolation and temporal extrapolation, while DeepONet exhibits larger deviations, particularly at later query times. For SWE, PDE-PFN better preserves the water-height profiles than DeepONet in the interpolation example and FNO in the temporal extrapolation example. These visual comparisons are consistent with the quantitative results reported in Tables~\ref{table:CDR_seen} and~\ref{table:SWE}.

\begin{figure}[h!]
  \centering
  \includegraphics[width=0.49\textwidth]{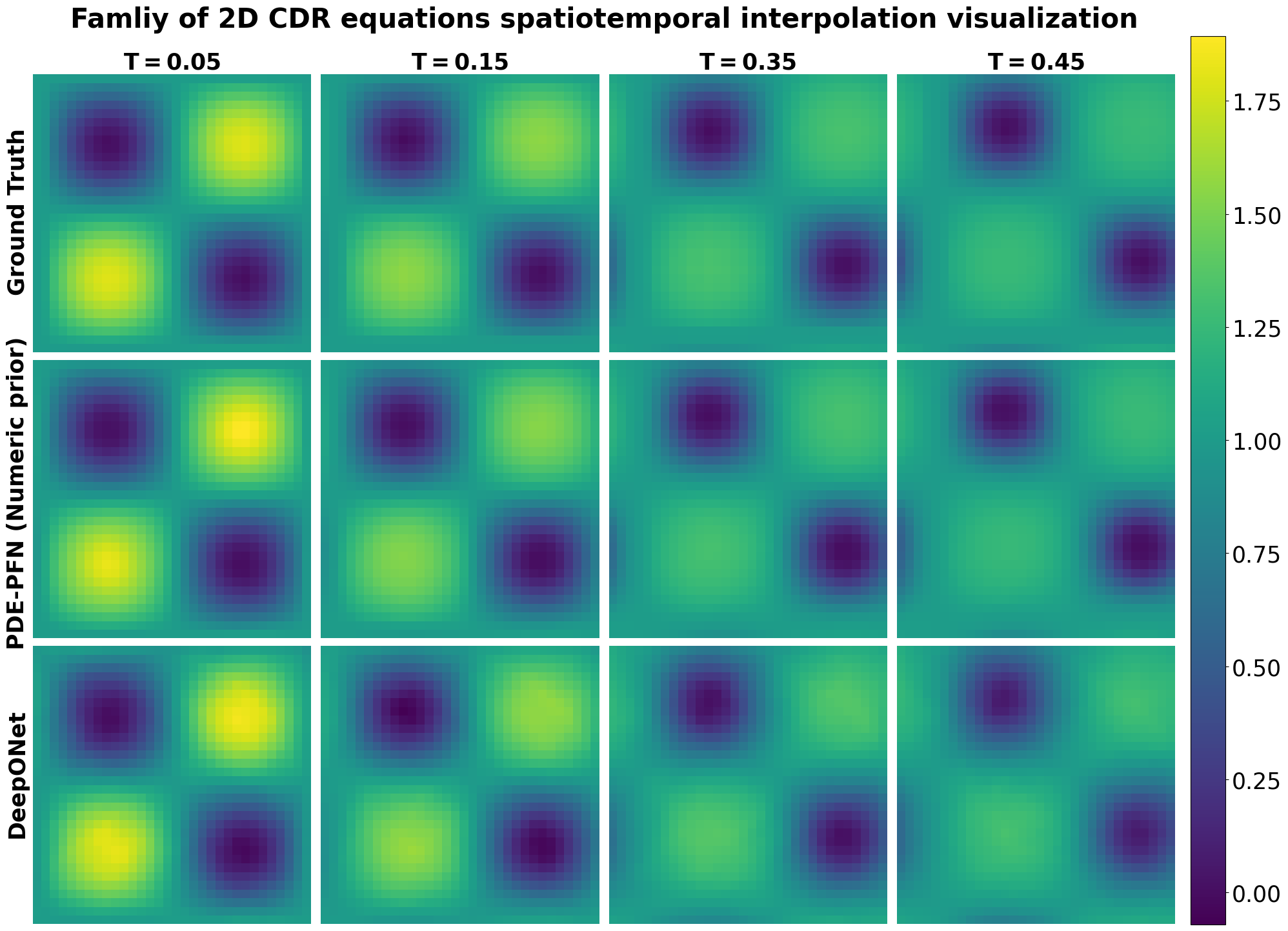}
  \includegraphics[width=0.49\textwidth]{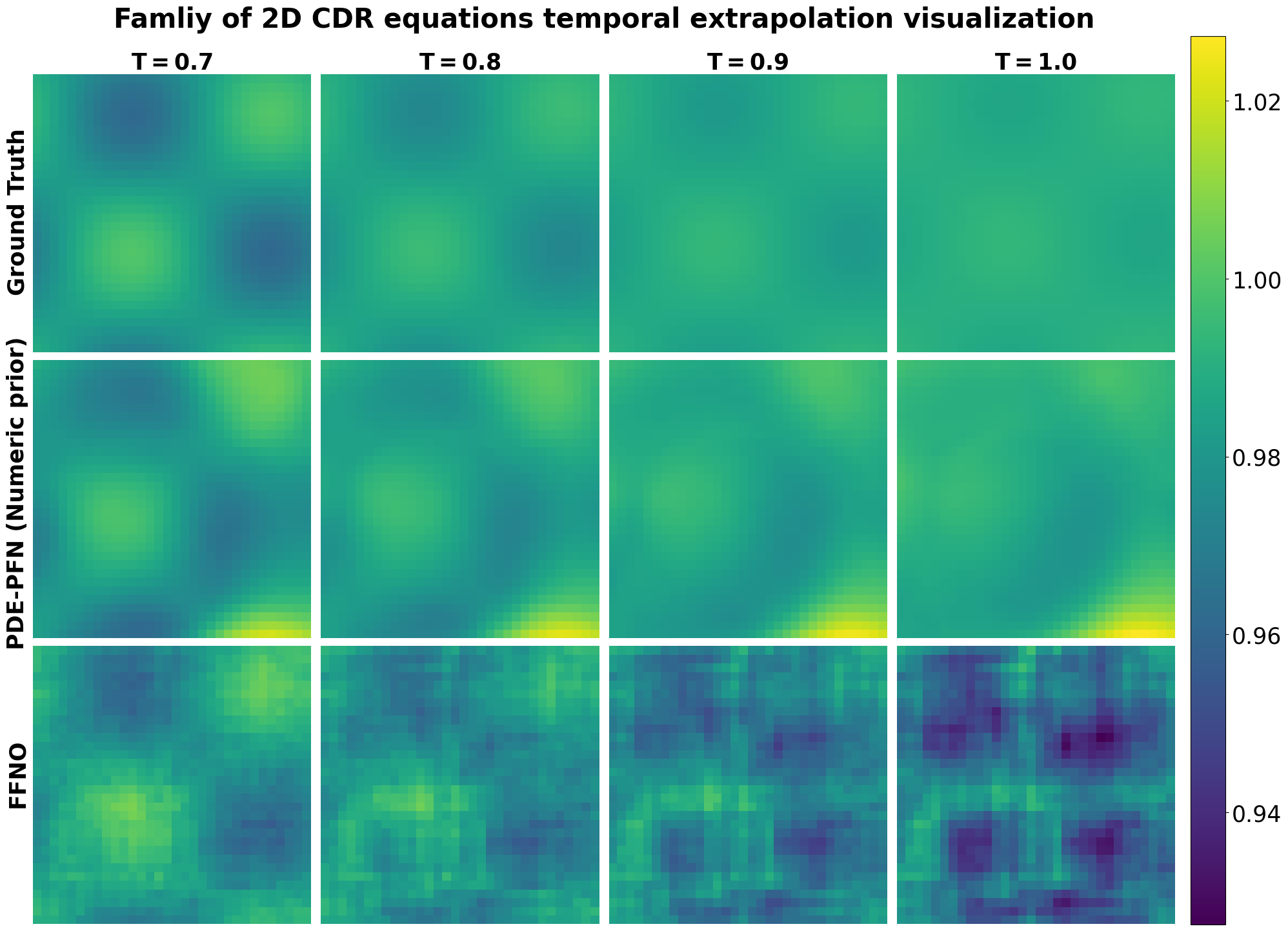}
  \caption{Visualization for the family of 2D CDR equation experiments. The left figure is for the spatiotemporal interpolation task and right figure is for the temporal extrapolation \label{visualization_2D_CDR}}
\end{figure}

\begin{figure}[h!]
  \centering
  \includegraphics[width=0.49\textwidth]{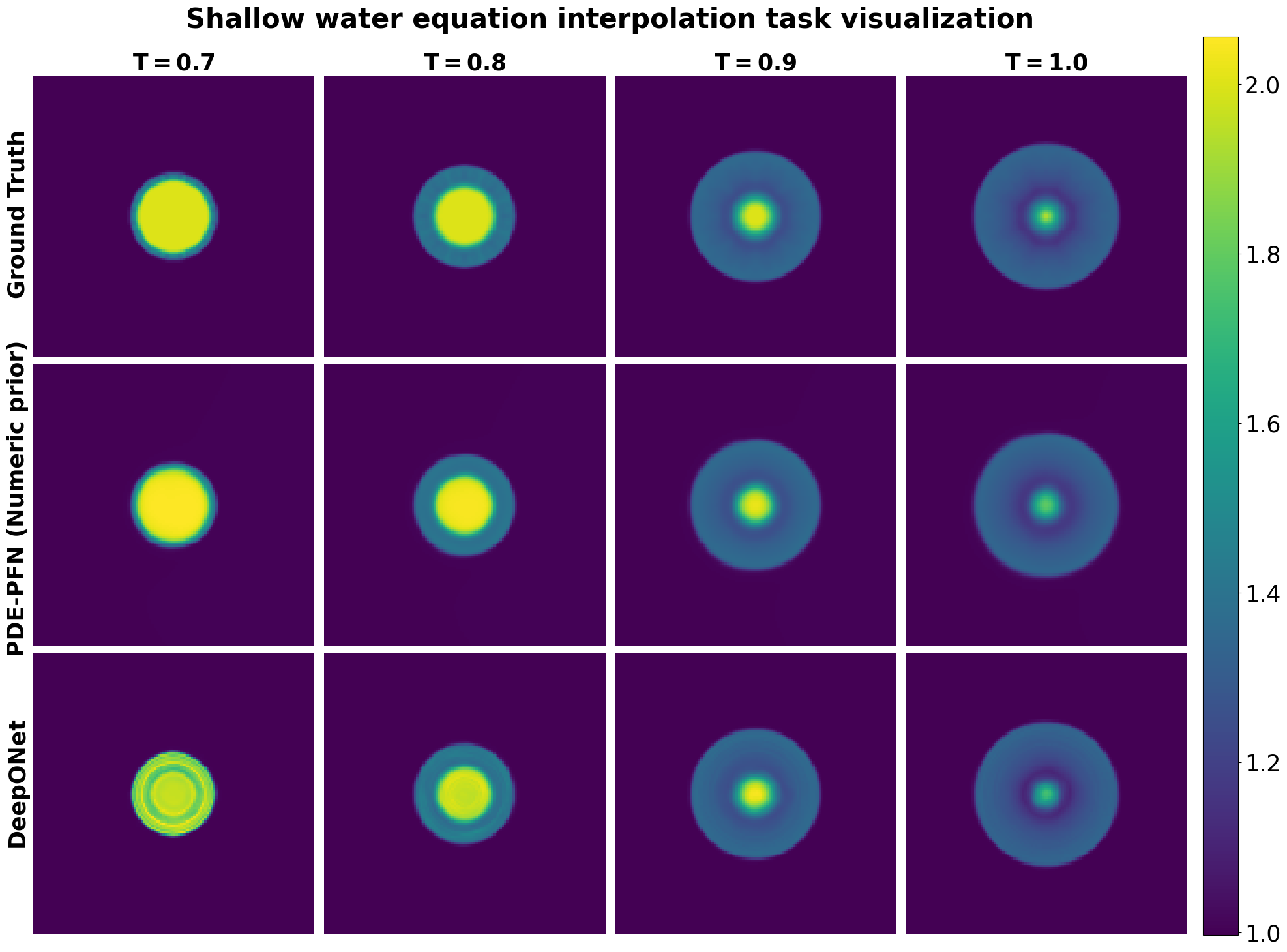}
  \includegraphics[width=0.49\textwidth]{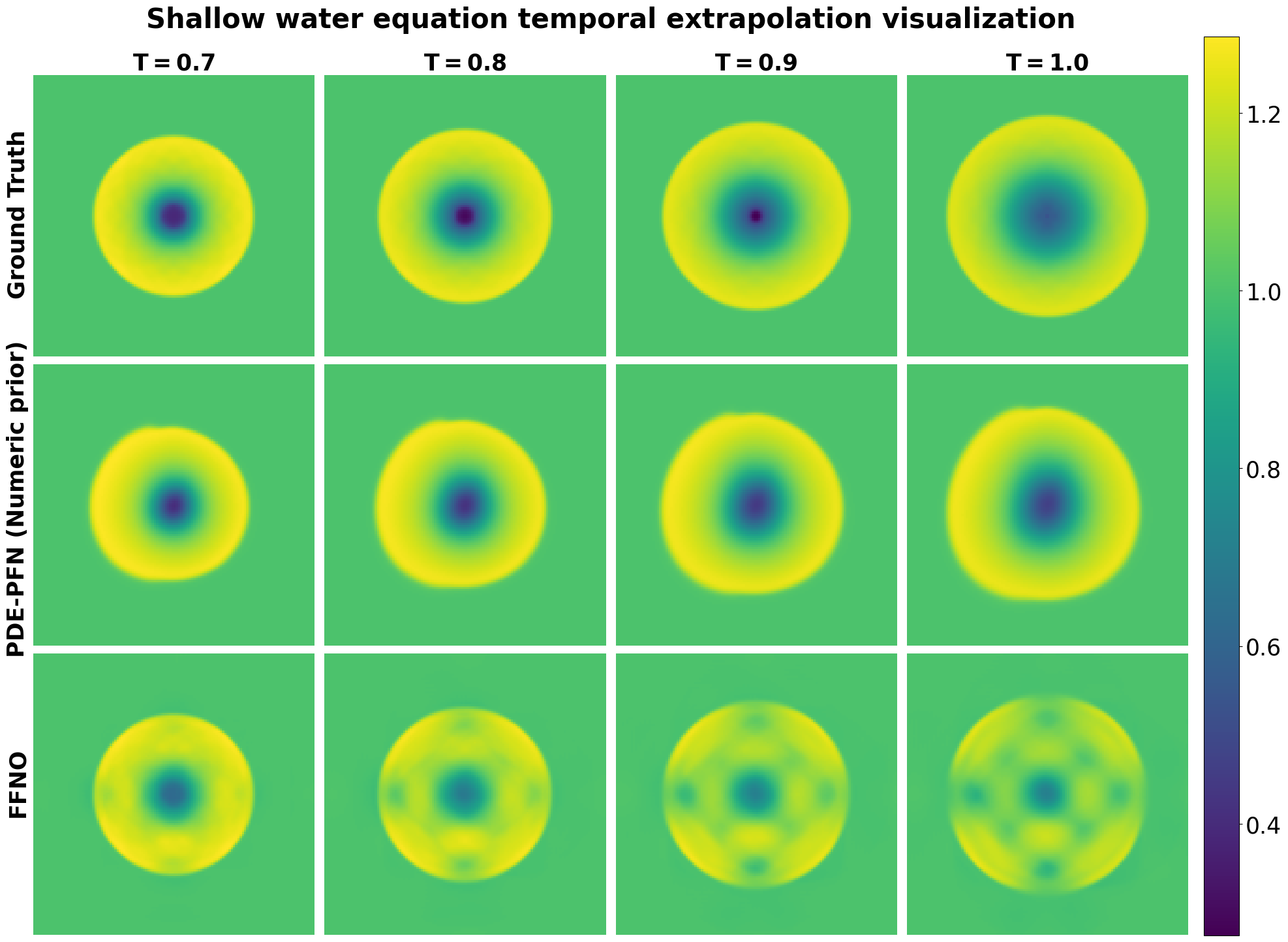}
  \caption{Visualization for the SWE experiments. The left figure is for the spatiotemporal interpolation task and right figure is for the temporal extrapolation \label{visualization_2D_SWE}}
\end{figure}

\section{Experimental Details} \label{app:experimental_details}

\subsection{Reproducibility} \label{app:reproducibility}
We are committed to ensuring the reproducibility of our research. All experimental procedures, data sources, and algorithms used in this study are clearly documented in the paper. The code will be provided as the supplementary material and be made publicly available upon publication, allowing others to validate our findings and build upon our work.

\subsection{Environments} \label{app:environment}
All the experiments were conducted using an NVIDIA RTX A6000. Details of the Python version and the packages used in the experiments can be found in the accompanying code’s environment file.

\subsection{Cost Comparison on the Main Experiments} \label{app:cost_comparison}
In addition to the comparison between the baselines in Table~\ref{table:baseline_comparison}, the additional comparison between the baselines is shown below in Table~\ref{table:experiment_cost_CDR}. In this tables, we compare the number of parameters, inference time, and the GPU memory usage for all models in their best settings. The inference time and GPU memory usage are measured in the test process on the family of 2D CDR equations temporal extrapolation task. In the experiment, both versions of our model are run with the same hyperparameters; therefore, we report the results under the single label Ours.

\begin{table}[ht]
\captionsetup{skip=1.5pt}
\caption{Cost comparisons between baselines and our model in the family of 2D CDR equations temporal extrapolation test process. \label{table:experiment_cost_CDR}}
\small
\centering
\resizebox{\textwidth}{!}{%
\begin{tabular}{lccc}
\toprule
\textbf{Model} & \textbf{Number of parameters}  & \textbf{inference time per sample(s)} &  \textbf{GPU memory usage(MB) } \\ 
\midrule
Ours     & 4,488,873   & 0.24 & 26.64    \\
DeepONet & 723,969     & 0.01 & 11.06    \\
FNO      & 143,268,097 & 0.10 & 1098.86  \\
F-FNO    & 23,631,105  & 0.02 & 98.47    \\
A-FNO    & 7,155,472   & 0.08 & 35.62    \\
LNO      & 5,337,857   & 0.08 & 24.85    \\
Poseidon & 157,625,930 & 0.27 & 2408.40  \\
DPOT     & 475,926,558 & 0.20 & 1825.72  \\			
\bottomrule
\end{tabular}%
}
\end{table}

\subsection{Hyperparameters} \label{app:hyperparameter}
In this section, we describe the number of data points for each equation and the hyperparameters used in the experiments for each model.
In Table~\ref{table:data_num}, the data points used in each experiment are shown separately for training, validation, and testing. For the training data, the context and query sets are constructed differently depending on the task. In the family of 2D CDR equations experiments with unseen coefficients, we only performed testing without training or validation; therefore, only the testing data points are reported.

\begin{table*}[htbp]
\caption{Number of data points used for training, validation, and testing in all experiments. \label{table:data_num}}
\small
\centering
\resizebox{\linewidth}{!}{%
\begin{tabular}{llccccc}
\toprule
\multirow{2}{*}{\textbf{PDE}} & \multirow{2}{*}{\textbf{Task}} & \multirow{2}{*}{\textbf{Train}} & \multicolumn{2}{c}{\textbf{Valid}}    &  \multicolumn{2}{c}{\textbf{Test}}\\ 
\cmidrule(lr){4-5} \cmidrule(lr){6-7}
 &  &  & \textbf{Context} & \textbf{Query} &  \textbf{Context} & \textbf{Query} \\ 
\midrule
\multirow{4}{*}{Family of 2D CDR equations} & Spatiotemporal interpolation & 6,144 & 1,024 & 1,024 & 1,024 & 4,096 \\
\cmidrule{2-7}
& Temporal extrapolation & 6,144 & 1,024 & 1,024 & 1,024 & 4,096  \\
\cmidrule{2-7}
& \begin{tabular}[l]{@{}l@{}} Unseen coeff \\ temporal extrapolation\end{tabular} & - & - & - & 1,024 & 4,096 \\
\midrule
\multirow{2}{*}{SWE} & Spatiotemporal interpolation & 98,304 & 16,384 & 16,384 & 16,384 & 65,536 \\ 
\cmidrule{2-7}
& Temporal extrapolation  & 98,304 & 16,384 & 16,384 & 16,384 & 65,536 \\
\midrule
CNSE &  \begin{tabular}[l]{@{}l@{}} Time trajectory predicting \\ operator learning \end{tabular} & 20,480 & 4,096 & 16,384 & 4,096 & 16,384 \\ 
\midrule
Airfoil & \begin{tabular}[l]{@{}l@{}} Time trajectory predicting \\ operator learning \end{tabular} & 26,165 & 5,233 & 20,932 & 5,233 & 20,932 \\ 
\midrule
Darcy Flow & \begin{tabular}[l]{@{}l@{}} $a \mapsto u$ \\ operator learning \end{tabular}  & \begin{tabular}[c]{@{}c@{}} $a$ : 16,384 \\ $u$ : 16,384\end{tabular} & \begin{tabular}[c]{@{}c@{}} $a$ : 16,384 \\ $u$ : 16,384\end{tabular} & \begin{tabular}[c]{@{}c@{}} $a$ : 16,384 \\ $u$ : 16,384\end{tabular} & \begin{tabular}[c]{@{}c@{}} $a$ : 16,384 \\ $u$ : 16,384 \end{tabular} & \begin{tabular}[c]{@{}c@{}} $a$ : 16,384 \\ $u$ : 16,384 \end{tabular} \\ 

\bottomrule
\end{tabular}%
}
\end{table*}

In the family of 2D CDR equation experiments, the evaluations on unseen coefficients are conducted without additional training. Therefore, we report only the hyperparameters used for the spatiotemporal interpolation and temporal extrapolation tasks trained on seen coefficients. During fine-tuning after pre-training, the first linear layer of the encoder and the final linear layer of the decoder are re-initialized to deal with cases where the number of variables to be predicted differs. The selected CDR, SWE, and CNSE configurations use FFE dimension zero, while Airfoil uses sinusoidal FFE and Darcy Flow uses FFT-based embeddings. For fairness, all best hyperparameters were selected based on those that achieved the highest performance on the validation. The hyperparameters for our model and the baselines in each experiment can be found in Table from~\ref{table:hyperparameter_CDR_SWE} to ~\ref{table:hyperparameter_Airfoil}. In each experiment, both versions of our model are run with the same hyperparameters; therefore, we report the results under the single label Ours.

\begin{table}[ht] 
\captionsetup{skip=1.5pt}
\caption{Best hyperparameter for each model used in the family of 2D CDR equations and SWE experiments. \label{table:hyperparameter_CDR_SWE}}
\small
\centering
\resizebox{\linewidth}{!}{%
\begin{tabular}{llcccc}
\toprule
\multirow{3}{*}{\textbf{Model}} & \multirow{3}{*}{\textbf{Hyperparameter Name}}  & \multicolumn{2}{c}{\textbf{Family of 2D CDR equations}} & \multicolumn{2}{c}{\textbf{SWE}}  \\ 
\cmidrule(lr){3-4} \cmidrule(lr){5-6}
& & \begin{tabular}[c]{@{}c@{}}\textbf{Spatiotemporal}\\ \textbf{interpolation}\end{tabular} & \begin{tabular}[c]{@{}c@{}}\textbf{Temporal}\\ \textbf{extrapolation}\end{tabular} & \begin{tabular}[c]{@{}c@{}}\textbf{Spatiotemporal}\\ \textbf{interpolation}\end{tabular} & \begin{tabular}[c]{@{}c@{}}\textbf{Temporal}\\ \textbf{extrapolation}\end{tabular} \\
\midrule
\multirow{6}{*}{Ours} & Attention layers number     & \multicolumn{4}{c}{16}  \\
                      & Attention hidden dimension  & \multicolumn{4}{c}{64}  \\
                      & Attention head number       & \multicolumn{4}{c}{5} \\
                      & FFE dimension               & \multicolumn{4}{c}{0} \\ 
                      & AGF depth                   & \multicolumn{4}{c}{15} \\ 
                      & En/Decoder hidden dimension & \multicolumn{4}{c}{512} \\ 
\midrule
\multirow{3}{*}{DeepONet} & Branch net depth   & 6   & 5   & 5   & 6   \\
                          & Trunk net depth    & 5   & 5   & 6   & 6   \\
                          & Hidden dimension   & 256 & 256 & 256 & 256   \\ \midrule
\multirow{2}{*}{FNO} & Layers number     & -  & 4   & -  & 3       \\
                     & Hidden dimension  & -  & 256 & -  & 128     \\ \midrule
\multirow{2}{*}{F-FNO} & Layers number     & -  & 5    & -  & 3    \\
                       & Hidden dimension  & -  & 256  & -  & 128  \\ \midrule
\multirow{2}{*}{A-FNO} & Layers number     & -  & 12   & -  & 3  \\
                       & Hidden dimension  & -  & 256  & -  & 256 \\
\midrule
\multirow{2}{*}{LNO}   & Layers number   & -  & 3   & -  & 3  \\
                       & block number    & -  & 8  & -  & 8 \\
\midrule
Poseidon & Pre-trained weight  & -  & L   & -  & T    \\
 \midrule
DPOT     & Pre-trained weight  & -  & L   & -  & M    \\
 \midrule
\multirow{2}{*}{PINN prior} & Training loss threshold  & \multicolumn{2}{c}{$1 \times 10^{-4}$} & \multicolumn{2}{c}{-} \\
                            & Maximum training epoch   & \multicolumn{2}{c}{200}  &\multicolumn{2}{c}{-}  \\
\bottomrule
\end{tabular}%
}
\end{table}

\begin{table}[htbp] 
\captionsetup{skip=1.5pt}
\caption{Best hyperparameter for each model used in the CNSE and Darcy Flow experiments. \label{table:hyperparameter_CNSE_Darcy}}
\small
\centering
\resizebox{0.7\linewidth}{!}{%
\begin{tabular}{llcc}
\toprule
\textbf{Model} & \textbf{Hyperparameter Name}  & CNSE & Darcy Flow \\
\midrule
\multirow{6}{*}{Ours} & Attention layers number     & \multicolumn{2}{c}{16}   \\
                      & Attention hidden dimension  & \multicolumn{2}{c}{64}   \\
                      & Attention head number       & \multicolumn{2}{c}{5}    \\
                      & FFE dimension               & 0  & 3    \\ 
                      & AGF depth                   & \multicolumn{2}{c}{15}   \\ 
                      & En/Decoder hidden dimension & \multicolumn{2}{c}{512}  \\ 
\midrule
\multirow{3}{*}{DeepONet} & Branch net depth   & 6    & 6     \\
                          & Trunk net depth    & 5    & 5     \\
                          & Hidden dimension   & 256  & 256   \\ \midrule
\multirow{2}{*}{FNO}      & Layers number      & 3    & 6     \\
                          & Hidden dimension   & 256  & 128   \\ \midrule
\multirow{2}{*}{F-FNO}    & Layers number      & 4    & 3     \\
                          & Hidden dimension   & 256  & 256   \\ \midrule
\multirow{2}{*}{A-FNO}    & Layers number      & 8    & 16    \\
                          & Hidden dimension   & 256  & 128   \\
\midrule
\multirow{2}{*}{LNO}      & Layers number     & 3    & 4   \\
                          & Block number      & 4    & 8   \\
\midrule
Poseidon  & pre-trained weight  & T   & T   \\
 \midrule
DPOT      & pre-trained weight  & L   & S   \\
\bottomrule
\end{tabular}%
}
\end{table}

\begin{table}[htbp] 
\captionsetup{skip=1.5pt}
\caption{Best hyperparameter for each model used in the Airfoil experiments. \label{table:hyperparameter_Airfoil}}
\small
\centering
\resizebox{0.7\linewidth}{!}{%
\begin{tabular}{llc}
\toprule
\textbf{Model} & \textbf{Hyperparameter Name}  & Operator Learining \\
\midrule
\multirow{6}{*}{Ours} & Attention layers number     & 16  \\
                      & Attention hidden dimension  & 64  \\
                      & Attention head number       & 5 \\
                      & FFE dimension               & 2 \\ 
                      & AGF depth                   & 15 \\ 
                      & En/Decoder hidden dimension & 512 \\ 
\midrule
\multirow{3}{*}{DeepONet} & Branch net depth   & 6    \\
                          & Trunk net depth    & 5    \\
                          & Hidden dimension   & 256  \\ \midrule
\multirow{2}{*}{Oformer}  & Layers number      & 3    \\
                          & Hidden dimension   & 64  \\
\bottomrule
\end{tabular}%
}
\end{table}


\section{Impact Statement}\label{app:Impact_Statement}
This research adheres to the ethical standards required for scientific inquiry. We have considered the potential societal impacts of our work and have found no clear negative implications. We also see no direct path from our model to malicious uses such as surveillance, disinformation, or privacy breaches. On the other hand, we recognize that improvements in PDE-solving capabilities could eventually be applied in real-world applications, such as climate modeling, biomedical simulations, or engineering systems.

All experiments were conducted in compliance with relevant laws and ethical guidelines, ensuring the integrity of our findings. We are committed to transparency and reproducibility in our research processes.

\end{document}